%% file: paper_aidata.tex
\documentclass{article}

\PassOptionsToPackage{numbers,sort&compress}{natbib}
\usepackage[preprint]{neurips_2026}
\makeatletter\renewcommand{\@noticestring}{}\makeatother 

\usepackage[utf8]{inputenc} 
\usepackage[T1]{fontenc}    
\usepackage{hyperref}       
\usepackage{url}            
\usepackage{booktabs}       
\usepackage{amsfonts}       
\usepackage{nicefrac}       
\usepackage{microtype}      
\usepackage{xcolor}         

\usepackage{graphicx}
\usepackage{float}          
\usepackage{listings}       

\usepackage{subcaption}
\usepackage{amsmath, amssymb, amsthm}
\usepackage{enumitem}
\usepackage{algorithm}
\usepackage{algpseudocode}
\usepackage{multirow}
\usepackage{makecell}
\usepackage{xspace}
\usepackage{tikz}
\usetikzlibrary{tikzmark}
\usepackage{eso-pic}        

\newtheorem{lemma}{Lemma}

\newtheorem{definition}{Definition}

\hypersetup{
    colorlinks,
    linkcolor={blue!80!black},
    citecolor={blue!80!black},
}

\AddToShipoutPictureFG*{%
  \AtPageUpperLeft{%
    \put(\LenToUnit{1.5in},\LenToUnit{-0.8in}){%
      \includegraphics[height=7mm,keepaspectratio]{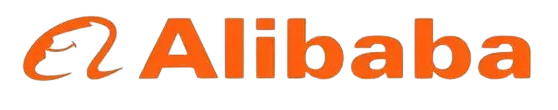}%
    }%
  }%
}

\title{OmniJudge or OmniBias? Diagnosing Multimodal Judges through Balanced, Decoupled Lenses}

\author{\mdseries
  Guangzheng Hu\textsuperscript{1,4} \quad
  Ziyue Jiang\textsuperscript{2,5} \quad
  Weixu Qiao\textsuperscript{1} \quad
  Lixin Zhang\textsuperscript{1} \quad
  Jianye Kang\textsuperscript{1} \\
  Yuru Wu\textsuperscript{1} \quad
  Rong Bao\textsuperscript{1} \quad
  Niantong Li\textsuperscript{1} \quad
  Wei Wang\textsuperscript{1} \quad
  Ziyi Cheng\textsuperscript{1} \\
  Xinfa Zhu\textsuperscript{2} \quad
  HangRui Hu\textsuperscript{1} \quad
  Ting He\textsuperscript{1} \quad
  Bing Zhao\textsuperscript{3} \quad
  Lin Qu\textsuperscript{1} \\
  Hu Wei\textsuperscript{1,\dag} \quad
  Jin Xu\textsuperscript{2,\ddag} \\[0.5em]
  \textsuperscript{1}Alibaba Group \quad
  \textsuperscript{2}Qwen Team \quad
  \textsuperscript{3}Alibaba DAMO Academy \\
  \textsuperscript{4}University of Melbourne \quad
  \textsuperscript{5}Zhejiang University \\[0.3em]
  \texttt{guangzhengh@student.unimelb.edu.au} \\
  \textsuperscript{\dag}\texttt{kongwang@alibaba-inc.com} \\
  \textsuperscript{\ddag}\texttt{renjun.xj@alibaba-inc.com} \\
}

\begin{document}

\maketitle

\vspace{-2em}
\begin{center}
  \href{https://github.com/SKYLENAGE-AI/D3OmniFramework}{\includegraphics[height=4mm]{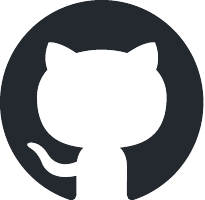}}\,
  \href{https://github.com/SKYLENAGE-AI/D3OmniFramework}{\texttt{https://github.com/SKYLENAGE-AI/D3OmniFramework}}
  \\
  \href{https://huggingface.co/datasets/skylenage-ai/D3OmniBench}{\includegraphics[height=4mm]{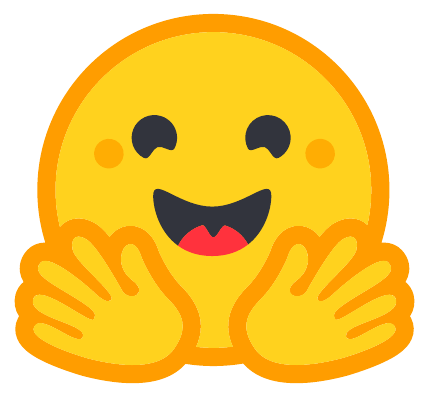}}\,
  \href{https://huggingface.co/datasets/skylenage-ai/D3OmniBench}{\texttt{https://huggingface.co/datasets/skylenage-ai/D3OmniBench}}
\end{center}
\vspace{-1em}

\begin{abstract}
Multimodal understanding models that can jointly judge text-to-image (T2I), text-to-video (T2V) and text-to-speech (TTS) generation are increasingly used as ``OmniJudges'' for evaluation and automatic annotation. How reliably they understand what they score remains unclear, since existing benchmarks and training data tend to overemphasize positive examples and to conflate distinct failure modes, so a judge may score well without recognizing failures while its capability gaps stay hidden. Motivated by this, we introduce \textbf{D\textsuperscript{3}-Omni}, a balanced and decoupled benchmark for diagnosing fine-grained multimodal understanding, covering $53$ orthogonal binary dimensions ($17$/$22$/$14$) and $10{,}671$ samples ($3{,}526$/$1{,}998$/$5{,}147$) across the three tasks. Rather than re-generating outputs, which may leak information across dimensions, we fix verified fully positive seeds and derive negatives through controlled prompt rewriting and atomic, dimension-isolating perturbations. The resulting \textbf{D\textsuperscript{3}} design is \textbf{Dual-balanced}, which helps alleviate negative-sample scarcity and per-dimension label imbalance; \textbf{Decoupled}, so that each error is attributable to a single capability; and \textbf{Dynamic}, steering construction toward under-represented regions of the label distribution as generative models improve. The suite reaches near $1{:}1$ per-dimension parity and a uniform distribution over all total-score levels. Under this balanced view, even strong OmniJudges tend to struggle on modality-related dimensions, to confirm satisfied requirements far more reliably than they detect violated ones, and to treat nominally distinct attributes as largely a single decision, suggesting that aggregate accuracy may hide systematic blind spots that a balanced and decoupled lens can help expose and, in turn, address.
\end{abstract}

\input{introduction}

\input{relatedwork}

\input{dimension}

\input{method}

\input{experiment}

\input{conclusion}

\bibliographystyle{unsrtnat}
\bibliography{main}

\appendix
\input{appendix}

\end{document}

%% file: introduction.tex
\section{Introduction}

Multimodal large language models (MLLMs) are rapidly evolving toward unified systems capable of understanding, reasoning, and generating across text, image, audio, video, and their combinations. Recent omni-modal benchmarks systematically assess these models' ability to understand and reason across visual, auditory, acoustic, and textual inputs~\cite{li2026omnibench,zhang2025omnieval}. Beyond serving as response generators, MLLMs are increasingly deployed as automatic annotators, preference judges, reward models, and distillation teachers; when tasked with cross-task scoring across T2I, T2V, and TTS, they are commonly referred to as ``OmniJudges''. However, the evaluation of these models in such roles, particularly as judge models (JMs) or reward models (RMs), remains considerably less developed than the assessment of their generation capabilities. Given their expanding real-world applications, addressing this evaluation gap has become increasingly critical.

\begin{figure}[!t]
\centering
\includegraphics[width=\textwidth]{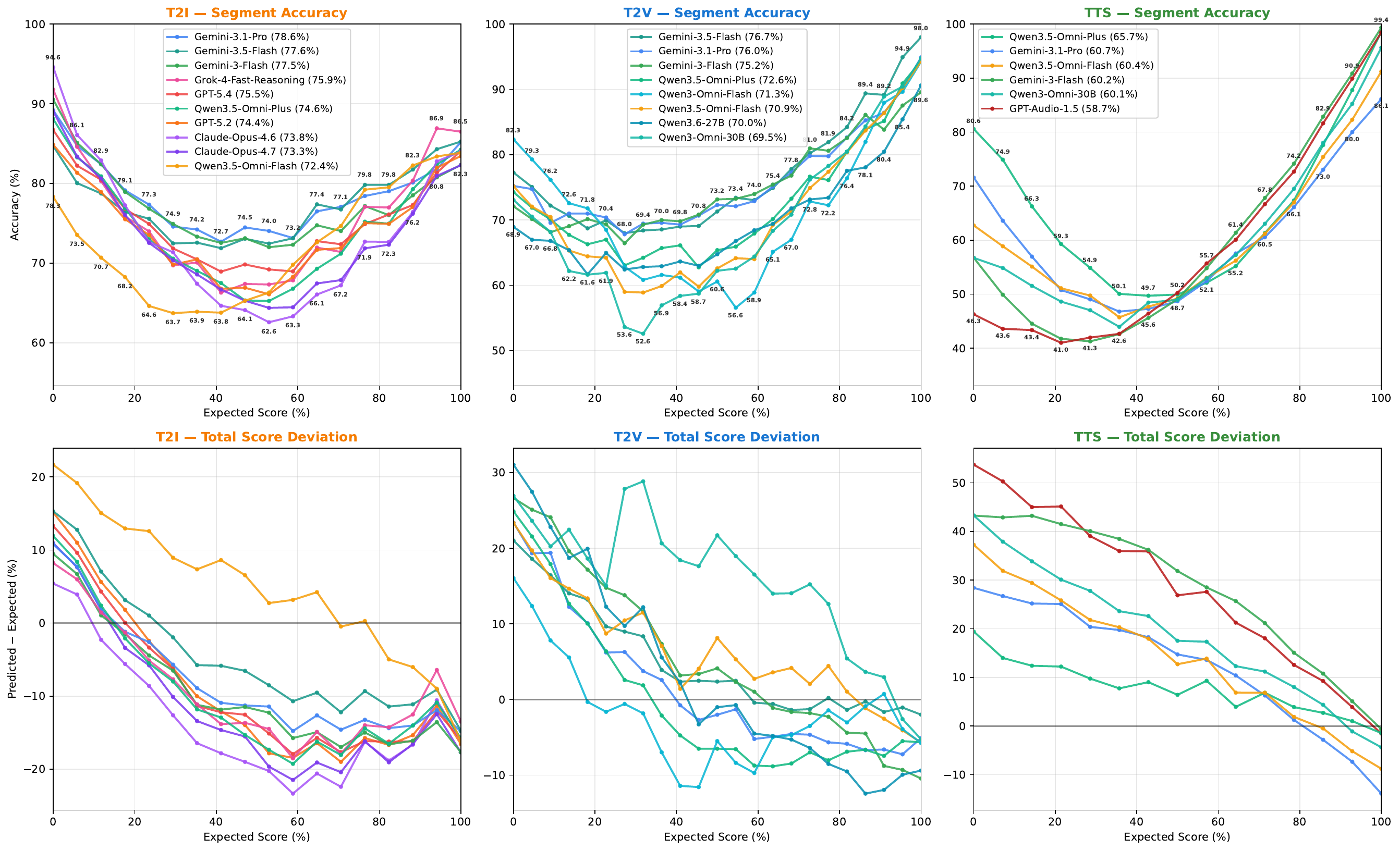}
\caption{\textbf{Top:} per-segment accuracy over the full ground-truth
score range; every judge dips in the mid-score regime, forming a U.
\textbf{Bottom:} total-score deviation (predicted $-$ expected):
positive $=$ Yes-bias, negative $=$ No-bias, zero $=$ calibrated.}
\label{fig:multi_total}
\end{figure}

Recent efforts evaluate LLM-as-a-judge systems, reward models, multimodal reward/preference models, and omni-modal reward models through dedicated benchmarks and preference datasets~\cite{zheng2023judging,lambert2025rewardbench,yasunaga2025multimodal,xiong2025llava,jin2025omni}. Many multimodal generation and evaluation benchmarks also rely on closed-source frontier models such as GPT-4o and Gemini for automated scoring or preference annotation~\cite{xiong2025llava,wijaya2024multimodal}. Although these protocols streamline evaluation, they reveal an important limitation: existing judge and reward benchmarks are often not explicitly designed around distributional balance. Samples are typically skewed across modalities, tasks, score levels, and evaluation criteria, with certain score ranges or quality categories represented disproportionately. Such distributional biases obscure whether high judging performance reflects genuine fine-grained multimodal understanding or simply memorization of dataset priors and majority-class patterns.

The need for balance is sharpened by a fundamental asymmetry between generation and judging. Real-world generations are predominantly acceptable outputs---catastrophic failures are, by construction, the minority---yet a judge is valuable precisely in the opposite regime, where it needs to catch the occasional error, the fine-grained mismatch, and the subtle factual or perceptual flaw inside an otherwise plausible sample. On a high-quality-dominated, imbalanced suite a judge can thus obtain a high score while remaining blind to the very failure modes that motivate its deployment: when $90\%$ of a dimension's samples are positive, a constant ``Yes'' already reaches $90\%$ accuracy without any genuine discrimination. A diagnostic omni-modal judge benchmark therefore needs two kinds of balance---score-level balance across quality intervals and dimension-level positive/negative parity---together with an evaluation taxonomy that is as decoupled as possible, so that failures can be attributed to specific abilities rather than conflated~\cite{chen2026advancing}.

Another challenge is the cost of benchmark construction. Although human annotation is important for reliability, constructing and updating a balanced and decoupled omni-modal benchmark entirely through manual labeling is expensive and inefficient. The cost becomes even higher when balance is required simultaneously across modalities, score levels, and fine-grained evaluation dimensions. This is particularly limiting in the Omni-LLM setting, where model capabilities evolve quickly and static benchmarks may soon lose their discriminative power. Therefore, instead of relying on exhaustive human annotation, benchmark construction should shift human effort toward targeted verification and quality control, while using automatic sample construction, model-assisted filtering, and consistency checking to support scalable data generation and dynamic benchmark updates~\cite{xiong2025llava,wijaya2024multimodal}.

Balanced test sets are, in fact, standard practice elsewhere: in long-tailed and imbalanced learning, models trained on skewed data are routinely evaluated on class-balanced test sets, precisely because accuracy is only meaningful once the label distribution is controlled. Why, then, does no explicitly balanced omni-modal judge benchmark yet exist? The obstacle is not a lack of motivation but the difficulty of construction: the quality of generated media is not directly controllable, the natural prevalence of each fine-grained dimension is intrinsically imbalanced, and building decoupled negatives that isolate a single dimension demands precise, controllable sample construction that ordinary data collection cannot provide. Breaking through this barrier is the core enabler of our work: rather than chasing controllable generations, we reverse-construct the prompt from a verified positive seed and apply controllable, dimension-isolating rewriting, coupled with a dynamic dual-balancing loop, which finally makes a balanced and decoupled evaluation suite feasible and exposes the capability blind spots that skewed benchmarks had kept hidden. Figure~\ref{fig:multi_total} previews one such blind spot: on our balanced suite every judge's per-segment accuracy collapses on the mixed-quality middle of the score range---a U-shaped dip that any aggregate number conceals.

To address these challenges, we propose \textbf{D\textsuperscript{3}-Omni}, a balanced and decoupled benchmark for diagnosing the fine-grained multimodal understanding of Omni-LLMs as JMs and RMs. Unlike existing multimodal benchmarks that mainly emphasize broad modality coverage or aggregate preference accuracy (e.g., \cite{lambert2025rewardbench,jin2025omni}), \textbf{D\textsuperscript{3}-Omni} focuses on distributional balance, decoupled assessment, and scalable construction. Specifically, it is designed to reduce bias induced by skewed score distributions, disentangle different judging dimensions for fine-grained diagnosis, and support low-human-cost benchmark construction and dynamic updates as Omni-LLM capabilities continue to evolve. In summary, \textbf{D\textsuperscript{3}-Omni} represents a paradigm shift from asking ``how well does the judge score'' to asking what the judge truly understands, and, equally importantly, what it fails to understand and evaluate.

\paragraph{Contributions.}
\textbf{Benchmark.} We introduce \textbf{D\textsuperscript{3}-Omni}, to the best of our knowledge the first balanced and decoupled omni-modal judge/reward benchmark for diagnosing the fine-grained multimodal understanding of Omni-LLMs, covering T2I, T2V and TTS with $10{,}671$ samples.
\textbf{Framework.} We propose the \textbf{D\textsuperscript{3}} framework---Dual-balanced, Decoupled and Dynamic---which controls the total-score distribution and the per-dimension positive/negative ratio jointly, reaching near $1\!:\!1$ parity on every dimension and a uniform distribution over all score levels.
\textbf{Taxonomy.} We design a decoupled taxonomy of $53$ orthogonal binary dimensions ($17$ for T2I, $22$ for T2V, $14$ for TTS) that separates prompt-related compliance (instruction following, attribute binding, spatial and temporal reasoning, text typesetting, audio-text alignment) from modality-related perceptual fidelity (visual realism, anatomical coherence, temporal stability, audio quality, speaker characteristics), so that each error is attributable to one capability.
\textbf{Pipeline.} We develop a low-human-cost, dynamically updatable construction pipeline combining automatic sample synthesis, model-assisted filtering, consistency checking and targeted human verification, allowing the benchmark to be extended as Omni-LLM capabilities evolve.
\textbf{Diagnosis.} Using \textbf{D\textsuperscript{3}-Omni} we surface blind spots that aggregate accuracy hides and that recur across model families: a U-shaped competence collapse on mixed-quality samples, a pervasive Yes-bias that widens from vision to speech, and pseudo-decoupling in which nominally orthogonal decisions collapse onto a single latent factor. Each shortcoming maps back to a concrete \textbf{D\textsuperscript{3}} operator, turning diagnosis into actionable data-side interventions; and all of it is legible only on a score- and dimension-balanced suite, since reading a single high-score segment alone would already reorder the leaderboard.

%% file: relatedwork.tex
\section{Related Work}

\subsection{Benchmarks for omni-modal generative tasks.}
A large body of work evaluates the generation quality of individual modalities: for text-to-image, compositional benchmarks such as T2I-CompBench~\citep{huang2023t2icompbench} and GenAI-Bench~\citep{lin2024genaibench} probe object presence, attribute binding, and spatial relations; for text-to-video, VBench~\citep{huang2024vbench}, EvalCrafter~\citep{liu2024evalcrafter}, and T2VScore~\citep{wu2024t2vscore} cover text-video alignment, visual and motion quality, and temporal consistency; and for text-to-speech, automatic evaluators predict perceptual quality, naturalness, intelligibility, and speaker similarity~\citep{lo2019mosnet,saeki2022utmos,maiti2023speechlmscore}. These benchmarks measure how well generators behave, but they treat the underlying evaluator as a trusted black box, leaving untested whether the evaluator itself reliably understands the fine-grained perceptual, alignment, and instruction-following properties it is asked to score. This gap motivates a closer look at the judge and reward models used for generative evaluation.

\subsection{Judge and reward models for generative content.}
LLM-as-a-judge has become a scalable alternative to human evaluation~\citep{zheng2023judging}, with text judges such as PandaLM, JudgeLM, Prometheus, and CompassJudger performing pairwise comparison, scalar scoring, and rubric-based critique~\citep{wang2024pandalm,zhu2023judgelm,li2023generative,kim2023prometheus,cao2024compassjudger}. In the generative-media domain, judges and reward models remain largely modality-specific: image reward models such as ImageReward, PickScore, and HPSv2 learn human preferences over generated images~\citep{xu2023imagereward,kirstain2023pick,wu2023human,zhang2024learning}, while speech and video rely on learned quality and alignment metrics~\citep{lo2019mosnet,saeki2022utmos,maiti2023speechlmscore,huang2024vbench,liu2024evalcrafter,wu2024t2vscore}. In parallel, generic omni-modal MLLMs are increasingly used as default judges across all three tasks---closed frontier systems (GPT-4o, GPT-5, Gemini, Claude Opus, Grok)~\citep{hurst2024gpt,openai2025gpt5,team2024gemini,google2025gemini3,claude35sonnet,anthropic2025claude4} and open omni-LLMs (Qwen3-Omni, MiniCPM-o, Mini-Omni2, Baichuan-Omni)~\citep{xu2025qwen3,cui2026minicpm,xie2024mini,li2024baichuan}. Yet none has been independently verified as a unified judge applying the same fine-grained criteria across T2I, T2V, and TTS: modality-specific evaluators approximate human preference scores and generic MLLMs are optimized for general task-solving, so in neither case is the judge itself tested for the perceptual, alignment, and instruction-following understanding that reliable evaluation requires. This motivates benchmarking the judges directly.

\subsection{Benchmarks for evaluating judge models.}
A growing line of work probes judge reliability. Text-only judge benchmarks test whether LLM judges can compare, score, or rank outputs across dialogue, instruction following, reasoning, coding, and safety~\citep{zheng2023judging,tan2024judgebench,lambert2025rewardbench,son2024mm}, revealing biases such as position and verbosity bias and weak sensitivity to factual correctness~\citep{zheng2023judging,tan2024judgebench,lambert2025rewardbench}. Recent multimodal and omni-modal benchmarks extend this to vision-language and any-modality judging~\citep{chen2024mllm,yasunaga2025multimodal,xiong2025llava,jin2025omni,pu2025judge,chen2026advancing,li2026omnibench,zhang2025omnieval}. Three limitations persist. First, they concentrate on text-only or vision-language settings, leaving speech and video judgment under-studied. Second, they score the final judgment directly without disentangling the underlying sub-abilities, so a failure cannot be attributed cleanly to perceptual misreading, prompt-alignment failure, or criterion misinterpretation. Third, their label distributions are inherited from naturally collected outputs and rarely controlled, so high accuracy may reflect label priors rather than understanding---which turns attention to how the labels themselves are distributed.

\subsection{Imbalanced label distribution undermines diagnostic capacity.}
Because most judge benchmarks are built from naturally collected outputs or human preferences~\citep{zheng2023judging,lambert2025rewardbench,chen2024mllm,pu2025judge}, the positive and negative labels within a fine-grained dimension (e.g., object presence, speaker emotion, temporal consistency, prompt faithfulness) are typically skewed. A judge can then post high accuracy by exploiting label priors---indeed a trivial majority-class baseline already does---so aggregate metrics fail to surface real capability gaps and cross-dimension comparisons become unreliable. The generation-judging asymmetry compounds this: real generations are themselves skewed toward acceptable samples, so a benchmark that mirrors this distribution leaves precisely the regime where a judge is most needed---catching the occasional error and the subtle perceptual or factual flaw---essentially untested. A diagnostic omni-modal judge benchmark therefore needs explicitly balanced per-dimension labels across text-to-speech, text-to-image, and text-to-video, so that it measures whether judges recognize both the presence and the absence of each fine-grained property rather than dataset bias.

%% file: dimension.tex
\section{Benchmark Dimension Design}
\label{sec:dimensions}

We aim to evaluate whether multimodal understanding models---so-called ``OmniJudges''---can make accurate, fine-grained, and \emph{decoupled} judgments about generated content. To this end, each benchmark instance is formalized as a triplet \((p, x, \mathbf{y})\), where \(p \in \mathcal{P}\) is an input prompt, \(x \in \mathcal{X}\) is a generated modality (image, video, or speech), and \(\mathbf{y} \in \{0,1\}^D\) is a ground-truth binary label vector. A judge model \(f\) predicts \(\hat{\mathbf{y}} = f(p, x)\), and its performance is measured by comparing \(\hat{\mathbf{y}}\) to \(\mathbf{y}\).

\subsection{Motivation}
A judge that produces only an aggregate quality or preference score reveals \emph{nothing} about which sub-ability is failing---exactly the diagnostic gap motivated in our introduction and related-work discussion. Yet judging generated content fundamentally requires two mechanistically distinct abilities: interpreting natural-language intent and perceiving low-level signal integrity. The first is a language-grounding problem operating jointly over \((p, x)\); the second is a perceptual-fidelity problem operating purely over \(x\). Folding them into a single rubric makes failure modes inseparable---a model that perfectly understands intent but cannot detect visual artifacts is, under aggregate scoring, indistinguishable from one with the opposite weakness. We therefore decompose every judgment into two disjoint categories, \emph{Prompt-related} (\(\mathcal{D}_p\)) and \emph{Modality-related} (\(\mathcal{D}_m\)), and inside each category we further refine the requirement into the smallest atomic dimension on which a controlled perturbation can flip the label \emph{without} affecting any other dimension. This recursive decoupling is the structural prerequisite for the targeted, dimension-isolating negative-sample construction in Section~\ref{sec:method}.

\subsection{Formalization and Two-Category Decomposition}
\label{subsec:dim_formalization}

Each dimension \(d \in \{1, \dots, D\}\) corresponds to a binary requirement of the form:  
``Does the generated output satisfy this specific requirement?''  
with answer \(y_d = 1\) (\texttt{Yes}) if satisfied and \(y_d = 0\) (\texttt{No}) otherwise. This formulation ensures judgments are unambiguous and directly attributable to specific capabilities.

The two categories are formally defined as follows.

\paragraph{Prompt-related dimensions (\(\mathcal{D}_p\)).}
These assess semantic alignment between the prompt and the output. They encompass requirements concerning subject identity, spatial and relational structure, compositional logic, stylistic intent, and adherence to explicit or implicit constraints in the prompt. For these dimensions, the ground-truth judgment depends on the joint interpretation of \(p\) and \(x\).

\paragraph{Modality-related dimensions (\(\mathcal{D}_m\)).}
These evaluate intrinsic properties of the generated signal that are independent of the prompt. They cover perceptual and structural qualities such as visual or audio fidelity, temporal coherence, physical plausibility, and low-level artifact presence. Here, the judgment is determined solely by \(x\), as it pertains to the modality's internal consistency and realism.

This categorization enables disentanglement of two fundamental judgment capabilities: understanding user intent versus perceiving signal integrity. Diagnosing which category exhibits systematic errors reveals whether a model's limitations lie in language grounding or perceptual robustness; for example, a judge that scores high on \(\mathcal{D}_p\) but consistently fails on \(\mathcal{D}_m\) is one that follows the rubric semantically yet remains blind to subtle perceptual flaws---an extremely common pattern that we document empirically in Section~\ref{sec:experiments}.

To support meaningful fine-grained analysis, all dimensions are designed to be \emph{orthogonal}: the satisfaction of any requirement should not depend on the state of others. For instance, whether an image correctly depicts the requested color (attribute binding) should be independent of whether the objects are in the correct spatial arrangement (spatial reasoning); a sample is allowed to fail on color while passing on layout, enabling precise localization of the deficit. The same orthogonality principle is applied \emph{recursively} to the fine-grained sub-dimensions inside each category: each requirement is the smallest atomic unit whose label can be flipped through a single controlled perturbation \emph{without} altering the label of any other dimension. This ensures that a failure on one dimension reflects a localized judgment deficit rather than a side effect of correlated attributes, and it is precisely what makes targeted negative-sample construction tractable downstream.

The total score is defined as the integer sum:
\[
s(\mathbf{y}) = \sum_{d=1}^D y_d \in \{0, 1, \dots, D\},
\]
which counts the number of satisfied requirements. Due to orthogonality, deviations from the maximum score can be precisely attributed to specific violated dimensions.

\subsection{Granularity and Coverage}
\label{subsec:dim_granularity}

Most existing automatic evaluators for generative media collapse ``quality'' or ``alignment'' into one or a handful of coarse metrics---FID, CLIP-similarity, MOS, single-axis preference scores, or a small set of generic Likert categories---which neither pinpoint specific failure types nor cover the breadth of attributes a modern OmniJudge is expected to inspect. Our taxonomy goes substantially finer and substantially broader: 17, 22, and 14 atomic binary requirements for T2I, T2V, and TTS respectively, totalling \(D = 53\) decoupled dimensions. Critically, every fine-grained sub-dimension is selected to target a specific failure mode that frontier generators \emph{still} occasionally exhibit despite producing visibly high-quality outputs overall---e.g., subtle text-rendering glitches and finger-anatomy errors in T2I, audio--video desynchronization and inter-frame flicker in T2V, or mismatched speaker personality and unstable volume in TTS. These rare-but-critical errors are exactly the cases a deployed judge is expected to detect reliably, yet they are also precisely the cases for which naturally collected data offers very few negative examples. Designing each requirement as small, atomic, and mutually orthogonal is therefore not a stylistic preference but a functional necessity: only such a taxonomy can isolate, balance, and diagnose each rare failure mode in turn, and only such a taxonomy admits dimension-specific negative samples whose construction does not leak into other dimensions.

\subsection{Instantiation Across Three Tasks}
\label{subsec:dim_instantiation}

We instantiate this design across three generation scenarios; the complete dimension lists, definitions, and category assignments are provided in Tables~\ref{tab:t2i_dimensions}, \ref{tab:t2v_dimensions}, and~\ref{tab:tts_dimensions} in the appendix.

\paragraph{Text-to-Image (T2I).} \(D = 17\): 14 prompt-related dimensions covering composition, attribute binding, spatial reasoning, typography, and negative-instruction handling, plus 3 modality-related dimensions evaluating material realism, edge clarity, and anatomical coherence; the latter three being precisely the residual perceptual flaws that even state-of-the-art diffusion models still leak through on otherwise impressive renderings.

\paragraph{Text-to-Video (T2V).} \(D = 22\): 16 prompt-related dimensions spanning subject, scene, lighting, style, and audio alignment, plus 6 modality-related dimensions targeting temporal stability, focal sharpness, motion rhythm, audio quality, and audio--video synchronization, where even strong T2V models still produce occasional glitches that aggregate quality scores routinely overlook.

\paragraph{Text-to-Speech (TTS).} \(D = 14\): 10 prompt-related dimensions covering textual and punctuation faithfulness together with eight separately-judged speaker characteristics (age, gender, personality, timbre, speed, pitch, tone, emotion), plus 4 modality-related dimensions on voice clarity, background noise, volume stability, and spectral integrity. This is a fine-grained decoupling of speaker attributes that current TTS systems often blur into a single ``looks-good'' voice.

However, real-world data lacks the balance and isolation required to evaluate this taxonomy reliably: outputs are skewed toward high scores, single-dimension negative examples are scarce, and conventional negative synthesis often violates orthogonality through unintended cross-dimensional effects. To address this, we introduce a dynamic dual-balanced construction framework (Section~\ref{sec:method}) that generates triplets \((p, x, \mathbf{y})\) with dimension-isolated negatives and near-exact balance across both per-dimension labels and total score levels.

%% file: method.tex
\section{Balanced Benchmark Construction}
\label{sec:method}

As established in Section~\ref{sec:dimensions}, reliable fine-grained
diagnosis of an omni-modal judge requires per-dimension label balance
and score-level uniformity over a taxonomy whose sub-dimensions
are mutually orthogonal. Achieving this in practice, however, is far
from trivial: directly resampling new outputs from a perturbed prompt
unavoidably introduces \emph{cross-dimensional leakage}: even minor
edits (e.g., ``running'' $\rightarrow$ ``walking'') drift the pose,
motion, background, or rendering quality, simultaneously corrupting
several supposedly independent dimensions and destroying the
orthogonality on which fine-grained attribution depends.

We therefore propose \textbf{D\textsuperscript{3}-Construction}, a
benchmark-construction pipeline organized around three pillars whose
names (Decoupling, Dual-Balancing, and Dynamic) all begin with the
letter D and together operationalise the design objectives of
Section~\ref{sec:dimensions}. Decoupling
(\S\ref{subsec:method_decoupling}) makes every generated sample
atomically attributable to a single dimension by fixing a fully
positive seed and altering only one factor at a time;
Dual-Balancing (\S\ref{subsec:method_dual_balancing}) gives the
precise mathematical objective that the resulting benchmark must
satisfy and proves it is realizable; Dynamic
(\S\ref{subsec:method_dynamic}) provides the iterative procedure that
actually drives any partial benchmark towards that objective.

\paragraph{Notation.}
Let $\tau \in \{\textsc{T2I}, \textsc{T2V}, \textsc{TTS}\}$ denote the
generation task, $\mathcal{D}_p^{(\tau)}$ and $\mathcal{D}_m^{(\tau)}$
the disjoint sets of prompt-related and modality-related
sub-dimensions, and
$D_\tau = |\mathcal{D}_p^{(\tau)}| + |\mathcal{D}_m^{(\tau)}|$ their
total size. The catalogues defined in Section~\ref{sec:dimensions}
yield $D_{\textsc{T2I}}\!=\!17$, $D_{\textsc{T2V}}\!=\!22$, and
$D_{\textsc{TTS}}\!=\!14$. Whenever the task is unambiguous we drop
the superscript $(\tau)$. Each benchmark example is a triple
\begin{equation}
\bigl(p,\, x,\, \mathbf{y}\bigr)\;\in\;\mathcal{P} \times \mathcal{X}_\tau \times \{0,1\}^{D_\tau},
\label{eq:benchmark_triple}
\end{equation}
where $p$ is a natural-language prompt, $x$ a generated modality
artifact (image, video, or speech waveform), and
$\mathbf{y} = (y_1, \dots, y_{D_\tau})$ a binary label vector aligned
positionally with the rubric in Section~\ref{sec:dimensions}:
$y_d = 1$ ($\textsc{Yes}$) iff $(p,x)$ satisfies the $d$-th rubric
requirement, and $y_d = 0$ ($\textsc{No}$) otherwise. The total score
of a sample is
$s(p,x,\mathbf{y}) = \sum_{d=1}^{D_\tau} y_d \in \{0, 1, \dots, D_\tau\}$,
and we write $s(\mathbf{y})$ when $(p,x)$ are clear from context.

\subsection{Decoupling: Atomic, Single-Dimension Negative Construction}
\label{subsec:method_decoupling}

\paragraph{Why decouple at the construction step.}
Aggregate quality on its own conveys nothing about which sub-ability
is failing; pinpointing the failing dimension requires negatives whose
violation is, by construction, attributable to one and only one
$d^*$. We obtain such atomically-attributable negatives by
decoupling modality generation from semantic mismatch creation:
rather than re-synthesizing a new modality, we fix a fully positive
seed $(p^+, x^+)$ and inject a single, dimension-isolated edit on
either the prompt side (for $d^*\!\in\!\mathcal{D}_p$) or the
modality side (for $d^*\!\in\!\mathcal{D}_m$).

\begin{definition}[Fully positive seed]
A pair $(p^+, x^+)$ is a \emph{fully positive seed} for task $\tau$ iff
its label vector under the rubric of Section~\ref{sec:dimensions} is
$\mathbf{y}^+ = \mathbf{1}_{D_\tau}$, i.e.\ every prompt-related and
every modality-related dimension is satisfied. We denote the set of
such seeds by
\[
\mathcal{S}_{\text{full}}^{(\tau)} \;=\; \bigl\{(p^+, x^+) : \mathbf{y}(p^+, x^+) = \mathbf{1}_{D_\tau}\bigr\}.
\]
\end{definition}

\begin{definition}[Single-dimension flip operator]
For a target dimension $d^* \in \mathcal{D}_p \cup \mathcal{D}_m$, a
\emph{single-dimension flip operator}
$\mathcal{F}_{d^*}\!:\!\mathcal{P}\!\times\!\mathcal{X}\!\to\!\mathcal{P}\!\times\!\mathcal{X}$ is any map satisfying
\begin{equation}
\mathbf{y}\bigl(\mathcal{F}_{d^*}(p^+,x^+)\bigr)_{d}
=
\begin{cases}
0, & d = d^*,\\
1, & d \in \{1,\dots,D_\tau\}\setminus\{d^*\},
\end{cases}
\quad \forall (p^+,x^+) \in \mathcal{S}_{\text{full}}.
\label{eq:flip_operator}
\end{equation}
$\mathcal{F}_{d^*}$ is \emph{prompt-side} if it modifies only $p^+$ and
\emph{modality-side} if it modifies only $x^+$.
\end{definition}

\paragraph{Step~1: Seed prompt construction with a tri-LLM modality-aware expert pipeline.}
A correct negative is meaningful only on top of a verifiably correct
seed. We therefore build $\mathcal{S}_{\text{full}}$ through a
modality-aware inverse-prompting procedure executed independently on
three large language models (the TTS / T2V / T2I captioning experts
implemented respectively as Qwen-3.5 Plus, Gemini-3.1 Flash, and
Gemini-3.1 Pro), each conditioned on a modality-specific inverse
template that explicitly enumerates every prompt-related sub-dimension
$d \in \mathcal{D}_p^{(\tau)}$ from Section~\ref{sec:dimensions}. The
three captions are reconciled by a human-in-the-loop calibration pass
that resolves disagreements and rewrites the seed prompt $p^+$ until
every sub-dimension $d$ in $\mathcal{D}_p$ is grounded by a literal,
modality-faithful requirement (e.g., for TTS the prompt explicitly specifies rate, volume, timbre, emotion, etc., matching what
the audio actually exhibits). Concretely, given the raw modality $x$,
\begin{equation}
p^+(x) \;=\; \textsc{Calibrate}\bigl(\,\textsc{Reconcile}(c_1, c_2, c_3),\; \mathcal{D}_p^{(\tau)},\; x\bigr),\quad c_i = \mathrm{LLM}_i^{(\tau)}\!\bigl(x;\,T^{(\tau)}_{\text{inv}}\bigr),
\end{equation}
where $T^{(\tau)}_{\text{inv}}$ is the modality-specific inverse
template. Modality-related labels for the same $x$ are then verified
by the same LLM ensemble (with conflicts again adjudicated by a human
auditor); we deliberately keep this verification of $\mathcal{D}_m$ manual
because a mis-certified seed would silently propagate as a false
negative once perturbed. Only samples whose certified labels equal
$\mathbf{1}_{D_\tau}$ enter $\mathcal{S}_{\text{full}}$.

\paragraph{Step~2: Prompt-side flips with graded counter-semantic rewriting.}
For $d^* \in \mathcal{D}_p$ we instantiate $\mathcal{F}_{d^*}$ as a
controlled prompt rewrite executed by Gemini-3.1 Pro, governed by:
(i)~\emph{minimal modification} (alter only tokens directly tied to
$d^*$); (ii)~\emph{dimensional isolation} (assert all other tokens that
ground dimensions $d \neq d^*$ remain unchanged); and
(iii)~\emph{structural coherence} (preserve fluency). To stress-test
judges across the full counter-semantic spectrum, the rewriter does
not flip in a single fixed direction; it samples a semantic-distance
level $\ell \in \mathcal{L}_{d^*}$ and emits a $p^-$ at that level.
For example, when the seed TTS prompt requires ``very fast''
speaking rate ($d^* = $ rate), the rewriter randomly draws among
$\ell\!\in\!\{\text{moderate-opposite: medium},\,\text{strong-opposite: slow},\,\text{maximal-opposite: very slow}\}$,
all of which still violate the rate dimension but at different
degrees of semantic divergence, ensuring that the resulting negatives
span a wide range of prompt--modality semantic distance and prevent
judges from over-fitting to one particular flip template. Formally,
the prompt-side flip is realized at a level $\ell$ drawn uniformly from
$\mathcal{L}_{d^*}$, giving $(p^-,x^+) = \mathcal{F}_{d^*}^{(\ell)}(p^+,x^+)$,
and the resulting negative example is
$(p^-, x^+, \mathbf{y}^{(d^*)})$ with $\mathbf{y}^{(d^*)}$ obtained
from $\mathbf{1}_{D_\tau}$ by setting the $d^*$-th coordinate to $0$.

\paragraph{Step~3: Modality-side flips via human-verified atomic operators.}
For $d^* \in \mathcal{D}_m$ a prompt edit cannot induce the failure;
we instead apply a modality-specific atomic operator
$\mathcal{O}_{d^*}\!:\!\mathcal{X}_\tau \to \mathcal{X}_\tau$ to the
seed modality $x^+$ (the prompt $p^+$ is held fixed). Each $x^+$
entering this stage has been manually inspected, ensuring that the
post-perturbation defect is the only reason a modality-related
dimension would flip. At a high level, the operators target three
families of perceptual defects: spatial / textural distortions
(e.g.\ frequency-domain blurring, anatomical deformation,
segmentation-driven warping), temporal / synchronization
perturbations (e.g.\ frame-order shuffling, $2\!\times$ speed masking,
audio--video offset), and audio-fidelity manipulations (e.g.\
spectrogram-domain bilateral filtering, additive reverberation,
gain modulation, band-targeted frequency manipulation). The full
list of atomic operators used in this work---6 for T2V, 3 for T2I,
and 4 for TTS---together with their implementation details is
provided in Appendix~\ref{app:atomic_operators}. For every
$(d^*, \tau)$ the corresponding operator $\mathcal{O}_{d^*}^{(\tau)}$
has been validated to leave all $d \neq d^*$ unchanged on a held-out
audit set, so that
$\tilde{x} = \mathcal{O}_{d^*}^{(\tau)}(x^+)$ paired with the
unaltered $p^+$ produces the desired triple
$(p^+, \tilde{x}, \mathbf{y}^{(d^*)})$ in
Eq.~\eqref{eq:benchmark_triple}.

\paragraph{Diagnostic metric for decoupling: the dimension-coupling matrix.}
To verify that the decoupling holds end-to-end (both in the rubric
itself and in any judge's predictions), we report the empirical
dimension-coupling matrix
$\Sigma \in \mathbb{R}^{D_\tau \times D_\tau}$ whose $(d, d')$ entry is
the Pearson correlation of the $\textsc{Yes}/\textsc{No}$ indicators
across the benchmark:
\begin{equation}
\Sigma_{d,d'} \;=\; \frac{\sum_{i=1}^{N}(y_{i,d} - \bar y_d)(y_{i,d'} - \bar y_{d'})}{\sqrt{\sum_{i=1}^{N}(y_{i,d} - \bar y_d)^2}\;\sqrt{\sum_{i=1}^{N}(y_{i,d'} - \bar y_{d'})^2}},
\qquad \bar y_d = \frac{1}{N}\sum_{i=1}^{N} y_{i,d}.
\label{eq:coupling}
\end{equation}
We compute $\Sigma$ on (a) the ground-truth labels (the reference
matrix) and on (b) each judge's predicted labels; substantial
off-diagonal mass in (b) that is absent from (a) is direct evidence
of spurious coupling, i.e., the judge cannot truly separate the
underlying capabilities, which is one of the four diagnostic axes
evaluated in Section~\ref{sec:experiments}.

\subsection{Dual-Balancing: Definition and Existence Theorem}
\label{subsec:method_dual_balancing}

\paragraph{The two balance constraints.}
A benchmark $\mathcal{B} = \{(p_i, x_i, \mathbf{y}_i)\}_{i=1}^{N}$ is
\emph{dually balanced} iff it is uniform along the score-segment axis
and along the per-dimension label axis. Writing $N = M(D_\tau + 1)$
with $M \in \mathbb{N}^+$ samples per score level, we require
\begin{align}
\bigl|\{i : s(\mathbf{y}_i) = s\}\bigr| &= M, &&\forall s \in \{0,1,\dots,D_\tau\}, \tag{Score uniformity}\label{eq:score_balance}\\
\sum_{i=1}^{N} y_{i,d} &= \tfrac{N}{2} \;=\; \tfrac{M(D_\tau+1)}{2}, &&\forall d \in \{1, \dots, D_\tau\}. \tag{Dimension parity}\label{eq:dim_balance}
\end{align}
Eq.~\eqref{eq:score_balance} enforces segment balance so that each
total-score bin is equally represented; Eq.~\eqref{eq:dim_balance}
enforces per-dimension balance so that no single sub-ability can be
solved by predicting the majority class. Importantly, neither
constraint alone is sufficient: a benchmark can be score-uniform yet
have one dimension always equal to $1$, and conversely a benchmark
can be per-dimension-balanced yet entirely concentrated near
$s\!\approx\!D_\tau/2$ (cf.\ \S\ref{subsec:segment} of the
experiments).

\begin{lemma}[Attainability of dual balance]
\label{thm:existence}
Fix a task $\tau$ and suppose
\begin{enumerate}
  \item[(A1)] $\mathcal{S}_{\textnormal{full}}^{(\tau)} \neq \varnothing$
  (fully positive seeds exist).
  \item[(A2)] For every $d \in \{1,\dots,D_\tau\}$, a single-dimension
  flip operator $\mathcal{F}_d$ satisfying Eq.~\eqref{eq:flip_operator}
  is realizable.
  \item[(A3)] For every $S \subseteq \{1,\dots,D_\tau\}$ the composition
  $\mathcal{F}_S = \bigcirc_{d \in S} \mathcal{F}_d$ acts as the
  characteristic flip
  $\mathbf{y}(\mathcal{F}_S(p^+,x^+))_d = \mathbf{1}[d \notin S]$ for
  every seed.
\end{enumerate}
Then for every tolerance $(\rho,\delta)$ with $\rho\in[0,1)$ and
$\delta\ge 0$ there is an achievable benchmark meeting the stopping
criterion~\eqref{eq:stop}; the target can be tightened to exact dual
balance (Eqs.~\eqref{eq:score_balance}--\eqref{eq:dim_balance}), which
is realizable with $|\mathcal{B}|=M(D_\tau+1)$ whenever $M(D_\tau+1)$ is
even. Consequently the three monitored quantities of
\S\ref{subsec:method_dynamic} can be driven below any prescribed
threshold, however strict.
\end{lemma}

\begin{proof}[Proof sketch]
Exact dual balance is the tightest instance ($\rho\to1$, $\delta=0$) and
dominates every looser tolerance, so it suffices to construct it. We do
so explicitly in Appendix~\ref{app:proof}: single-dimension flips are
composed via (A2)+(A3) into multi-dimension flips whose $M(D_\tau+1)$
label vectors realise each score level exactly $M$ times and load every
dimension exactly $N/2$ times, establishing
Eqs.~\eqref{eq:score_balance}--\eqref{eq:dim_balance} when
$M(D_\tau+1)$ is even; any looser $(\rho,\delta)$ is then met a fortiori.
\end{proof}

\subsection{Dynamic: Iterative Construction Algorithm}
\label{subsec:method_dynamic}

Theorem~\ref{thm:existence} guarantees that a dually balanced benchmark
exists, but in practice the seed pool, the LLM-rewriter, and the
modality operators all incur finite costs, so we cannot enumerate the
combinatorial template once and stop. Instead, we grow $\mathcal{B}$
dynamically: at each step we identify the least-populated non-full score,
the most positive-dominated dimensions, and the least-loaded seed, and
append one controlled multi-dimension negative for that target.

\paragraph{Monitored quantities.}
Write $N_t = |\mathcal{B}^{(t)}|$ and let
$\sigma(i)\in\mathcal{S}_{\textnormal{full}}$ be the seed of sample $i$.
Over $\mathcal{B}^{(t)}$ we track exactly the three quantities that
define the target: the per-dimension positive and negative proportions,
the score histogram, and the per-seed negative load,
\begin{equation}
\begin{aligned}
\pi_d^{\textsc{yes}} &= \frac{1}{N_t}\sum_{i} \mathbf{1}[y_{i,d}=1],\qquad
\pi_d^{\textsc{no}} = 1-\pi_d^{\textsc{yes}},\\[2pt]
N_s^{(t)} &= \bigl|\{i: s(\mathbf{y}_i)=s\}\bigr|,\qquad
u_\sigma^{(t)} = \bigl|\{i: \sigma(i)=\sigma,\ s(\mathbf{y}_i)<D_\tau\}\bigr|.
\end{aligned}
\label{eq:monitor}
\end{equation}

\paragraph{Per-step target.}
Each step addresses the least-populated non-full score level, the most
positive-dominated dimensions, and the least-loaded seed:
\begin{equation}
s^* = \arg\min_{0\le s < D_\tau} N_s^{(t)},\qquad
T^* \in \operatorname*{arg\,min}_{|T|=D_\tau-s^*}\ \sum_{d\in T} \frac{\pi_d^{\textsc{no}}}{\pi_d^{\textsc{yes}}},\qquad
\sigma^* = \arg\min_{\sigma\in\mathcal{S}_{\textnormal{full}}} u_\sigma^{(t)},
\label{eq:targets}
\end{equation}
with ties broken uniformly at random, so that $T^*$ gathers the
$D_\tau-s^*$ dimensions of smallest negative/positive ratio. Applying
the composite flip $\mathcal{F}_{T^*}$ of
\S\ref{subsec:method_decoupling} to $\sigma^*=(p^+,x^+)$---its
prompt-related coordinates $T^*\cap\mathcal{D}_p$ by graded rewriting
and its modality-related coordinates $T^*\cap\mathcal{D}_m$ by the
atomic operators---produces a negative of score $s^*$ with label
$\mathbf{y}^{(T^*)}$, where $\mathbf{y}^{(T^*)}_d=\mathbf{1}[d\notin T^*]$,
which is appended to $\mathcal{B}^{(t+1)}$.

\paragraph{Stopping criterion.}
The loop halts once the two balance axes are within tolerance:
\begin{equation}
\min_{d}\, \frac{\pi_d^{\textsc{no}}}{\pi_d^{\textsc{yes}}} > \rho
\qquad\text{and}\qquad
\max_s N_s^{(t)} - \min_s N_s^{(t)} < \delta .
\label{eq:stop}
\end{equation}
The first condition forces every dimension to a negative/positive ratio
of at least $\rho$ (perfect balance is ratio $1$); the second flattens
the score histogram; the seed load $u_\sigma$ enters only through the
least-loaded-seed choice in Eq.~\eqref{eq:targets}, keeping negatives
spread across seeds rather than being a stopping target. For all three
tasks of D\textsuperscript{3}-Omni we use $\rho=0.95$ and $\delta=6$.
Each step raises the negative count of the currently most
positive-dominated dimensions and of the least-filled score, so both
deficits decrease monotonically; by Lemma~\ref{thm:existence} the
tolerance region is attainable for any $(\rho,\delta)$---indeed down to
exact balance---so the loop terminates.

\begin{algorithm}[t]
\caption{D\textsuperscript{3}-Construction: dynamic dual-balanced benchmark assembly}
\label{alg:construction}
\begin{algorithmic}[1]
\Require Raw modalities $\mathcal{X}$; $\mathcal{D}_p, \mathcal{D}_m$; tri-LLM ensemble $\{\mathrm{LLM}_1,\mathrm{LLM}_2,\mathrm{LLM}_3\}$; reconciler / rewriter Gemini-3.1 Pro; modality operators $\{\mathcal{O}_d\}_{d \in \mathcal{D}_m}$; tolerances $\rho,\delta$.
\Ensure Benchmark $\mathcal{B}$ meeting the stopping criterion Eq.~\eqref{eq:stop}.
\Statex \textbf{Stage~A: Seed curation (Decoupling, Step~1).}
\State Run inverse-prompting on $\mathcal{X}$ with each $\mathrm{LLM}_i$, reconcile + human-calibrate to obtain $p^+(x)$ for every $x$.
\State Run modality-integrity check on every $x$ via $\{\mathrm{LLM}_i\}$ + human audit.
\State $\mathcal{S}_{\text{full}} \gets \{(p^+(x), x) : \mathbf{y}(p^+(x), x) = \mathbf{1}_{D_\tau}\}$;\; $\mathcal{B} \gets \mathcal{S}_{\text{full}}$.
\Statex \textbf{Stage~B: Monitor--generate--update loop.}
\State Compute $\pi_d^{\textsc{no}}/\pi_d^{\textsc{yes}},\, N_s,\, u_\sigma$ over $\mathcal{B}$. \Comment{Eq.~\eqref{eq:monitor}}
\While{$\min_d \pi_d^{\textsc{no}}/\pi_d^{\textsc{yes}} \le \rho$ \textbf{or} $\max_s N_s - \min_s N_s \ge \delta$} \Comment{Eq.~\eqref{eq:stop}}
  \State $s^* \gets \arg\min_{s<D_\tau} N_s$;\; $T^* \in \operatorname*{arg\,min}_{|T|=D_\tau-s^*}\sum_{d\in T}\pi_d^{\textsc{no}}/\pi_d^{\textsc{yes}}$;\; $\sigma^* \gets \arg\min_\sigma u_\sigma$. \Comment{Eq.~\eqref{eq:targets}}
  \State $(p^-,x^-) \gets \mathcal{F}_{T^*}(\sigma^*)$: prompt coords of $T^*$ by rewriting, modality coords by operators;\; $\mathbf{y}^- \gets \mathbf{y}^{(T^*)}$.
  \State $\mathcal{B} \gets \mathcal{B} \cup \{(p^-, x^-, \mathbf{y}^-)\}$;\; update $\pi_d^{\textsc{no}}/\pi_d^{\textsc{yes}}, N_s, u_\sigma$.
\EndWhile
\State \Return $\mathcal{B}$.
\end{algorithmic}
\end{algorithm}

The output of Algorithm~\ref{alg:construction} is a benchmark meeting
the stopping criterion Eq.~\eqref{eq:stop}---the attainable target
guaranteed by Lemma~\ref{thm:existence}---in the canonical triplet
format of Eq.~\eqref{eq:benchmark_triple}, equipped with the diagnostic
metrics of \S\ref{subsec:method_dual_balancing}; this is exactly the
benchmark evaluated by the experiments in Section~\ref{sec:experiments}.

\subsection{Diagnostic Metrics Derived from the Dual Balance}
\label{subsec:method_metrics}

The dual balance is precisely what makes the metrics below readable
as capability signals rather than as artifacts of an imbalanced label
distribution. Let $\hat{\mathbf{y}}_i \in \{0,1\}^{D_\tau}$ be a
judge's predicted label vector for the $i$-th benchmark sample, and
let $\mathcal{B}_s = \{i : s(\mathbf{y}_i) = s\}$.

\paragraph{Metric 1: Per-dimension accuracy.}
This metric quantifies the judge's resolution of the $d$-th
sub-ability:
\begin{equation}
\textrm{Acc}_d \;=\; \frac{1}{N}\sum_{i=1}^{N} \mathbf{1}\bigl[\hat{y}_{i,d} = y_{i,d}\bigr].
\label{eq:acc_dim}
\end{equation}

\paragraph{Metric 2: Per-segment accuracy.}
This metric is the mean per-sample dimension match inside score
bin $s$:
\begin{equation}
\textrm{Acc}_s \;=\; \frac{1}{|\mathcal{B}_s|}\sum_{i \in \mathcal{B}_s}\frac{1}{D_\tau}\sum_{d=1}^{D_\tau} \mathbf{1}\bigl[\hat{y}_{i,d} = y_{i,d}\bigr].
\label{eq:acc_seg}
\end{equation}
Under Eq.~\eqref{eq:score_balance} every $|\mathcal{B}_s|$ equals $M$,
so the curve $s \mapsto \textrm{Acc}_s$ is undistorted by segment
mass.

\paragraph{Metric 3: Per-segment perfect-match rate.}
This metric reports the fraction of samples in segment $s$ whose
predicted label vector exactly matches the ground truth:
\begin{equation}
\textrm{PM}_s \;=\; \frac{1}{|\mathcal{B}_s|}\sum_{i \in \mathcal{B}_s}\mathbf{1}\!\Bigl[\hat{\mathbf{y}}_i = \mathbf{y}_i\Bigr].
\end{equation}

\paragraph{Metric 4: Dimension-restricted Yes/No exact-match rates.}
These metrics report respectively the fraction of samples on which
a judge gets every ground-truth-\textsc{Yes} dimension right, and
analogously for \textsc{No}:
\begin{align}
\textrm{PM}^{\textsc{Yes}} &= \frac{1}{|\mathcal{I}^{\textsc{Yes}}|}\sum_{i \in \mathcal{I}^{\textsc{Yes}}}\mathbf{1}\!\Bigl[\hat{y}_{i,d} = 1,\; \forall d \in \mathcal{Y}^+_i\Bigr], &&\mathcal{Y}^+_i = \{d : y_{i,d}=1\},\\
\textrm{PM}^{\textsc{No}}\; &= \frac{1}{|\mathcal{I}^{\textsc{No}}|}\sum_{i \in \mathcal{I}^{\textsc{No}}}\mathbf{1}\!\Bigl[\hat{y}_{i,d} = 0,\; \forall d \in \mathcal{Y}^-_i\Bigr], &&\mathcal{Y}^-_i\, = \{d : y_{i,d}=0\},
\label{eq:yesno_pm}
\end{align}
where $\mathcal{I}^{\textsc{Yes}} = \{i : \mathcal{Y}^+_i \neq \varnothing\}$
and $\mathcal{I}^{\textsc{No}} = \{i : \mathcal{Y}^-_i \neq \varnothing\}$.

\paragraph{Metric 5: Dimension-coupling matrix.}
$\Sigma$ (Eq.~\eqref{eq:coupling}) is evaluated separately on
$\{\mathbf{y}_i\}$ (reference) and on $\{\hat{\mathbf{y}}_i\}$
(predicted), with off-diagonal excess
$\bigl|\Sigma^{\text{pred}}_{d,d'}\bigr| - \bigl|\Sigma^{\text{ref}}_{d,d'}\bigr|$
flagging spurious coupling.

These five quantities, taken jointly, instantiate the four
diagnostic axes listed in Section~\ref{sec:experiments}: segment
robustness (Eq.~\eqref{eq:acc_seg}), dimension-wise resolution
(Eq.~\eqref{eq:acc_dim}), dimensional decoupling
(Eq.~\eqref{eq:coupling}), and Yes/No commit behavior
(Eq.~\eqref{eq:yesno_pm}).

%% file: experiment.tex
\section{Experiments}
\label{sec:experiments}

We evaluate a suite of state-of-the-art multimodal judges on
D\textsuperscript{3}-Omni and organise the diagnosis around five
questions that an unbalanced benchmark cannot answer:
(i)~at the macro level, do today's OmniJudges differ enough to be
distinguished, and does the ranking transfer across modalities
(Sec.~\ref{subsec:setup})?
(ii)~as the ground-truth total score varies, in which score regime
does each judge break down, and in which direction does it miscalibrate
(Sec.~\ref{subsec:segment})?
(iii)~when forced to commit on all-Yes or all-No segments, does a judge
expose a structured Yes/No class bias, and does its aggregate accuracy
actually reflect an ability to localize the few defects that a segment
contains (Sec.~\ref{subsec:yesno})?
(iv)~when nominally orthogonal dimensions are to be judged
independently, does a judge keep its per-dimension decisions decoupled
or let one decision drag the others (Sec.~\ref{subsec:coupling})?
(v)~at the per-dimension level, which fine-grained requirements floor
every judge near chance, and which expose the largest inter-model gaps
(Sec.~\ref{subsec:per_dim})?
Because D\textsuperscript{3}-Omni is balanced over both dimensions and
total scores, every drop in accuracy is attributable to a localized
capability gap rather than to a skewed label prior, turning aggregate
numbers into a transparent diagnostic map. The diagnosis is built
around four figures
(Figures~\ref{fig:multi_total},~\ref{fig:perfect},~\ref{fig:corr}
and~\ref{fig:dim}); Sec.~\ref{subsec:synthesis} then maps the four
diagnosed shortcomings back to the three operators of the
D\textsuperscript{3} construction framework (Sec.~\ref{sec:method}).
The full judging prompts are reproduced in
Appendix~\ref{app:eval_prompts}.

\subsection{Setup and Overall Accuracy}
\label{subsec:setup}

\paragraph{Judges, tasks and metrics.}
We benchmark ten judges on T2I, eight on T2V and six on TTS, drawn
from the Gemini, GPT, Claude, Grok and Qwen families; we treat them
under a single pool, since the goal is to characterize judging
behavior rather than to compare access types. Every judge of a given
task is queried with the same prompt template
(Appendix~\ref{app:eval_prompts}) and answers in a JSON array of $D$ Yes/No decisions. Unless stated otherwise, the observations below refer to the judges evaluated here. All judges are queried at temperature $0$ with every
other decoding parameter left at its provider default, so the reported
behavior reflects each model's deterministic mode rather than sampling
variance. Because every dimension is a binary judgment, we adopt
four complementary metrics:
(i)~\textbf{per-dimension binary accuracy}, the primary metric;
(ii)~\textbf{segment-wise accuracy} on the total-score slice
$s(\mathbf{y})\!=\!\sum_d y_d\!\in\!\{0,\dots,D\}$, with one curve per
judge per task, together with the total-score calibration curve;
(iii)~the \textbf{off-diagonal Pearson correlation matrix} of each
judge's predicted per-dimension decisions, which measures whether the
judge decides each near-orthogonal dimension independently or allows one latent factor to drive several outputs (a pair is flagged at
$\rho\!>\!0.6$ and treated as strongly entangled at $\rho\!>\!0.8$);
and (iv)~\textbf{per-segment Yes/No perfect-match rates}, which expose
class-prior shortcuts and, more strictly, whether a judge can localize
every minority defect within a segment.

\paragraph{Overall accuracy hides the diagnosis.}
Averaging per-dimension binary accuracy over the entire test set
gives a tight band on the visual tasks: the
ten T2I judges sit between $72.4\%$ and $78.6\%$ and the eight T2V
judges between $69.5\%$ and $76.7\%$, a $6$--$7$ point spread that one
would dismiss as a saturating leaderboard. TTS does not so much widen
the band as reorder it: accuracy compresses to $58.7\%$--$65.7\%$ and
the family that leads the visual tasks no longer leads, with
Qwen-3.5-Omni-Plus on top ($65.70\%$) and the two Gemini judges,
dominant on T2I and T2V, sliding into the lower half. The four
diagnostic axes that follow show that the visual-task tightness is
itself an artifact: every T2I/T2V judge loses substantial accuracy on
specific fine-grained regimes, but each one loses it in a different
regime, so the macro average conceals these losses.

\subsection{Score-Segment Behavior and Calibration}
\label{subsec:segment}

\begin{table}[t]
\centering
\footnotesize
\setlength{\tabcolsep}{3pt}
\resizebox{\textwidth}{!}{%
\begin{tabular}{@{}lrrrrrrrrrrrrrrr@{}}
\toprule
& \multicolumn{5}{c}{T2I} & \multicolumn{5}{c}{T2V} & \multicolumn{5}{c}{TTS} \\
\cmidrule(lr){2-6}\cmidrule(lr){7-11}\cmidrule(l){12-16}
Judge & $\mathcal{D}_p$ & $\mathcal{D}_m$ & Low & Mid & High & $\mathcal{D}_p$ & $\mathcal{D}_m$ & Low & Mid & High & $\mathcal{D}_p$ & $\mathcal{D}_m$ & Low & Mid & High \\
\midrule
Gemini-3.1-Pro      & $\mathbf{83.21}$ & $56.93$ & $\underline{81.2}$ & $\mathbf{74.2}$ & $\underline{80.3}$ & $\mathbf{85.20}$ & $51.41$ & $\mathbf{71.1}$ & $\underline{71.7}$ & $\mathbf{84.5}$ & $\underline{64.93}$ & $50.16$ & $\underline{58.4}$ & $50.5$ & $73.2$ \\
Gemini-3-Flash      & $\underline{81.91}$ & $\underline{57.17}$ & $\mathbf{81.5}$ & $\underline{73.0}$ & $78.1$ & $\underline{84.81}$ & $49.52$ & $\underline{69.3}$ & $\mathbf{72.3}$ & $\underline{83.5}$ & $63.09$ & $\underline{52.98}$ & $46.8$ & $\underline{50.7}$ & $\mathbf{83.0}$ \\
Qwen-3.5-Omni-Plus  & $78.27$ & $\mathbf{57.65}$ & $78.8$ & $67.2$ & $77.9$ & $80.30$ & $\underline{52.04}$ & $68.1$ & $66.3$ & $82.6$ & $\mathbf{66.33}$ & $\mathbf{64.20}$ & $\mathbf{67.2}$ & $\mathbf{51.4}$ & $\underline{78.5}$ \\
Qwen-3.5-Omni-Flash & $76.64$ & $52.74$ & $69.9$ & $66.9$ & $\mathbf{80.5}$ & $77.85$ & $\mathbf{52.27}$ & $66.2$ & $63.1$ & $82.3$ & $64.02$ & $51.54$ & $55.5$ & $50.4$ & $75.4$ \\
\midrule
Gemini-3.5-Flash    & $\mathbf{82.85}$ & $52.92$ & $78.0$ & $\mathbf{73.4}$ & $\mathbf{81.3}$ & $\mathbf{85.57}$ & $\underline{53.30}$ & $\underline{71.3}$ & $\mathbf{71.3}$ & $\mathbf{87.0}$ & -- & -- & -- & -- & -- \\
GPT-5.4             & $\underline{80.20}$ & $53.70$ & $78.8$ & $\underline{70.0}$ & $77.7$ & -- & -- & -- & -- & -- & -- & -- & -- & -- & -- \\
Grok-4-Fast         & $80.00$ & $\underline{56.88}$ & $\underline{79.4}$ & $68.5$ & $\underline{79.9}$ & -- & -- & -- & -- & -- & -- & -- & -- & -- & -- \\
GPT-5.2             & $78.96$ & $52.94$ & $77.3$ & $68.4$ & $77.4$ & -- & -- & -- & -- & -- & -- & -- & -- & -- & -- \\
Claude-Opus-4.7     & $77.18$ & $55.57$ & $78.7$ & $66.1$ & $75.2$ & -- & -- & -- & -- & -- & -- & -- & -- & -- & -- \\
Claude-Opus-4.6     & $77.10$ & $\mathbf{58.60}$ & $\mathbf{80.8}$ & $64.7$ & $76.0$ & -- & -- & -- & -- & -- & -- & -- & -- & -- & -- \\
Qwen-3.6-27B        & -- & -- & -- & -- & -- & $\underline{77.76}$ & $49.16$ & $65.0$ & $\underline{65.6}$ & $78.8$ & -- & -- & -- & -- & -- \\
Qwen-3-Omni-Flash   & -- & -- & -- & -- & -- & $76.96$ & $\mathbf{56.18}$ & $\mathbf{71.8}$ & $60.4$ & $80.3$ & -- & -- & -- & -- & -- \\
Qwen-3-Omni-30B     & -- & -- & -- & -- & -- & $76.92$ & $49.53$ & $63.0$ & $61.6$ & $\underline{83.0}$ & $\mathbf{64.35}$ & $\underline{49.62}$ & $\mathbf{51.8}$ & $\underline{50.3}$ & $\underline{78.3}$ \\
GPT-Audio-1.5       & -- & -- & -- & -- & -- & -- & -- & -- & -- & -- & $\underline{62.18}$ & $\mathbf{50.25}$ & $\underline{43.2}$ & $\mathbf{51.0}$ & $\mathbf{81.9}$ \\
\bottomrule
\end{tabular}}
\caption{Per-judge accuracy with the three tasks side by side. Top four
rows: cross-task judges present in all tasks; below: task-specific
judges (\texttt{--}: not evaluated). For each task,
$\mathcal{D}_p$/$\mathcal{D}_m$ are the prompt-/modality-related macro
accuracies and Low/Mid/High the accuracies in the
$0$--$33$/$33$--$66$/$66$--$100\%$ score bands. Per column,
\textbf{bold} marks the best and \underline{underline} the second-best value,
computed separately within the cross-task and task-specific blocks.}
\label{tab:dp_dm_macro}
\end{table}

\paragraph{Reading the two rows.}
We slice the test set by the ground-truth total score
$s(\mathbf{y})\!=\!\sum_d y_d$, which counts how many of the $D$
requirements a sample satisfies; by construction every score segment
holds the same number of samples. The top row of
Figure~\ref{fig:multi_total} plots, for each judge, its per-dimension
binary accuracy within each segment (one curve per judge, the legend
value being the overall mean), and thus reveals the difficulty regime
on which a judge decides accurately or breaks down. The bottom row
plots the total-score deviation, predicted minus expected total, so a
point above the zero baseline is over-scoring and one below is
under-scoring, revealing the direction in which the judge's scores are
biased. Read together, the two rows turn a single aggregate accuracy
into a difficulty-resolved capability curve and a bias direction.

\paragraph{Top row: a U-shaped collapse on the mixed-quality middle.}
Every curve is markedly U-shaped: a judge decides easily on the
all-violated and all-satisfied extremes but loses substantial
per-dimension accuracy in the middle, where roughly half the
requirements hold and the other half fail, i.e.\ the ambiguous
mixed-quality samples on which no all-Yes or all-No prior helps and each dimension has to be decided on its own. On T2I, GPT-5.2 reads
$84.9\%$ at $s\!=\!0$ and $83.4\%$ at $s\!=\!D$ yet drops to about
$66\%$ in the middle; at the aggregate level Gemini-3.1-Pro leads T2I
($78.6\%$), Gemini-3.5-Flash leads T2V ($76.7\%$), and TTS is hardest, where the best judge, Qwen-3.5-Omni-Plus, reaches only $65.7\%$. The dip marks a judge's
real weakness and sits precisely where ordinary benchmarks are
sparsest, so an unbalanced set yields a high mean from the two easy tails while the collapse remains hidden.
Table~\ref{tab:dp_dm_macro} gives the banded counterpart (Low/Mid/High
$=$ the $0$--$33$/$33$--$66$/$66$--$100\%$ score bands): for almost every
judge the Mid band is the trough, e.g.\ Gemini-3.1-Pro on T2I reads
$81.2/74.2/80.3$ and Qwen-3.5-Omni-Plus $78.8/67.2/77.9$.

\paragraph{The shape of the dip is itself a quality profile.}
Judges differ in the depth, position and symmetry of the U, and each
shape reads as a distinct competence profile. Gemini-3.1-Pro traces the
shallowest, flattest U on T2I (trough $72.7\%$, endpoints $89.2\%$ and
$85.2\%$), a judge equally at ease on clearly-bad, mixed and
clearly-good samples, which is why it also tops the macro table; GPT-5.2
and Gemini-3.5-Flash are near-symmetric, concentrating their error
squarely on the half-satisfied middle. Claude-Opus-4.6 is the opposite
extreme, a deep narrow valley (trough $62.6\%$) whose all-violated tail
is the highest of any judge ($94.6\%$) but whose recovery is weak, the
profile of a judge that only reliably flags blatant defects, while
Grok-4-Fast-Reasoning shows high shoulders over a deep trough, so
reading only its extremes would overstate it. The lone T2I exception is
Qwen-3.5-Omni-Flash: instead of a symmetric U it traces a rising slope
whose trough sits early (near $29\%$) and whose all-satisfied end
($83.9\%$) exceeds its all-violated end ($78.3\%$), the only judge more
reliable on good samples than on blatantly bad ones.

\paragraph{Endpoint asymmetry: confirming quality versus detecting
defects.}
Beyond depth, the relative height of the two arms is diagnostic: the
difference $\mathrm{hi}\!-\!\mathrm{lo}$ between the all-satisfied and
all-violated ends says whether a judge is better at confirming good
content or at catching bad content. Averaged over judges the sign
flips with the modality: T2I is mildly negative ($-3.67$;
Claude-Opus-4.6 reaches $-10.6$), so judges are stricter on
poor images, whereas T2V ($+19.05$) and TTS ($+32.45$) are strongly
positive, so on the temporal tasks judges confirm all-good content well
but cannot
flag every defect on all-bad content; the asymmetry is same-signed and
slightly larger in Chinese (T2V $+21.22$, TTS $+36.81$). The trough
migrates accordingly, deepening and shifting toward the low-score,
mostly-violated side on the temporal tasks (the three T2V Gemini judges
bottom out near $27\%$, and Qwen-3-Omni-30B reaches $52.57\%$, one of
the deepest collapses in the study), so the hardest samples change from
half-right-half-wrong on T2I to mostly-wrong-but-not-all on T2V and
TTS. The extreme is GPT-Audio-1.5 on TTS, closer to a monotonic ramp
than a U, rising from $46.3\%$ on all-violated speech to $98.5\%$ on
all-satisfied speech and essentially saying Yes until the audio is
clearly clean.

\paragraph{Bottom row: a regression-to-the-mean bias.}
The deviation curve falls monotonically with the ground-truth score,
giving a consistent pattern of over-scoring the bad and under-scoring
the good: every judge sits above the zero baseline at $s\!=\!0$ and
below it at $s\!=\!D$, crossing zero somewhere in the low-to-mid range.
On T2I, GPT-5.2 over-scores by $+15$ points on all-violated samples and
under-scores by $-17$ on all-satisfied ones, so the error is
directional and predictable rather than random. The bias is most severe
on the all-violated tail of the temporal tasks: even the best-calibrated
TTS judge, Qwen-3.5-Omni-Plus, still mislabels about one in five No
dimensions as Yes at $s\!=\!0$, and GPT-Audio-1.5 mislabels more than
half ($46.3\%$ accuracy there). This is the macro-level footprint of
the Yes-biased class prior that Sec.~\ref{subsec:yesno} isolates
dimension by dimension, where no TTS judge produces all-No on more than
$12.24\%$ of the samples for which all-No is correct.

\paragraph{Why only a balanced set reveals this, and a counterfactual.}
Both the U trough and the sign-changing deviation line can be read only
when every score segment is populated. A direct counterfactual makes
the point: reading a single high-score segment would rewrite the
leaderboard, since the all-satisfied segment alone would rank
Gemini-3-Flash first on TTS ($99.4\%$), a judge third from last under
balanced macro accuracy ($60.2\%$), while the all-violated segment would
instead promote Qwen-3.5-Omni-Plus; on T2I the all-violated segment would
elevate Claude-Opus-4.6 ($94.6\%$), eighth of ten on the macro average.
Score uniformity is therefore a necessary condition for a trustworthy
leaderboard, not a stylistic preference.

\paragraph{Cross-modal synthesis.}
Every judge converges on the same strategy: exploit a majority-class
shortcut wherever the labels concentrate, and pay for it in the
mixed-quality middle. T2I expresses this as a mildly low-biased U,
T2V and TTS as a high-biased U whose trough migrates toward the
low-score side and deepens with the length of the temporal signal, so
the amplitude grows from T2I through T2V to TTS. Importantly, the
rebound of these curves on high-score segments is a per-dimension
average and should not be read as evidence that a judge has localized
the few remaining defects; Sec.~\ref{subsec:yesno} shows that this
rebound is largely carried by confirming the many satisfied dimensions.

\subsection{Yes/No Perfect Match: Does the Headline Number Localise
Defects?}
\label{subsec:yesno}

\paragraph{Reading the two rows.}
Per-dimension accuracy averages over all dimensions of a sample, so a
judge can post a high mid-to-high score by confirming the many
satisfied dimensions while missing the few violated ones.
Figure~\ref{fig:perfect} applies a stricter test per score segment: the
top row is the Yes-rate, the fraction of samples on which the
judge is correct on \emph{all} ground-truth-Yes dimensions, and the
bottom row is the symmetric No-rate on all ground-truth-No
dimensions. Because getting every dimension of one polarity right is
demanding, the absolute rates sit far below the segment accuracy of
Figure~\ref{fig:multi_total}, so the figure is read for relative
contrasts, between judges and between the Yes and No rows, rather than
for absolute height. The No row is the strict measure of defect
attribution: a high-score sample keeps only one or two No dimensions,
and No-rate asks whether the judge catches every one of them.

\begin{figure*}[!t]
\centering
\includegraphics[width=\textwidth]{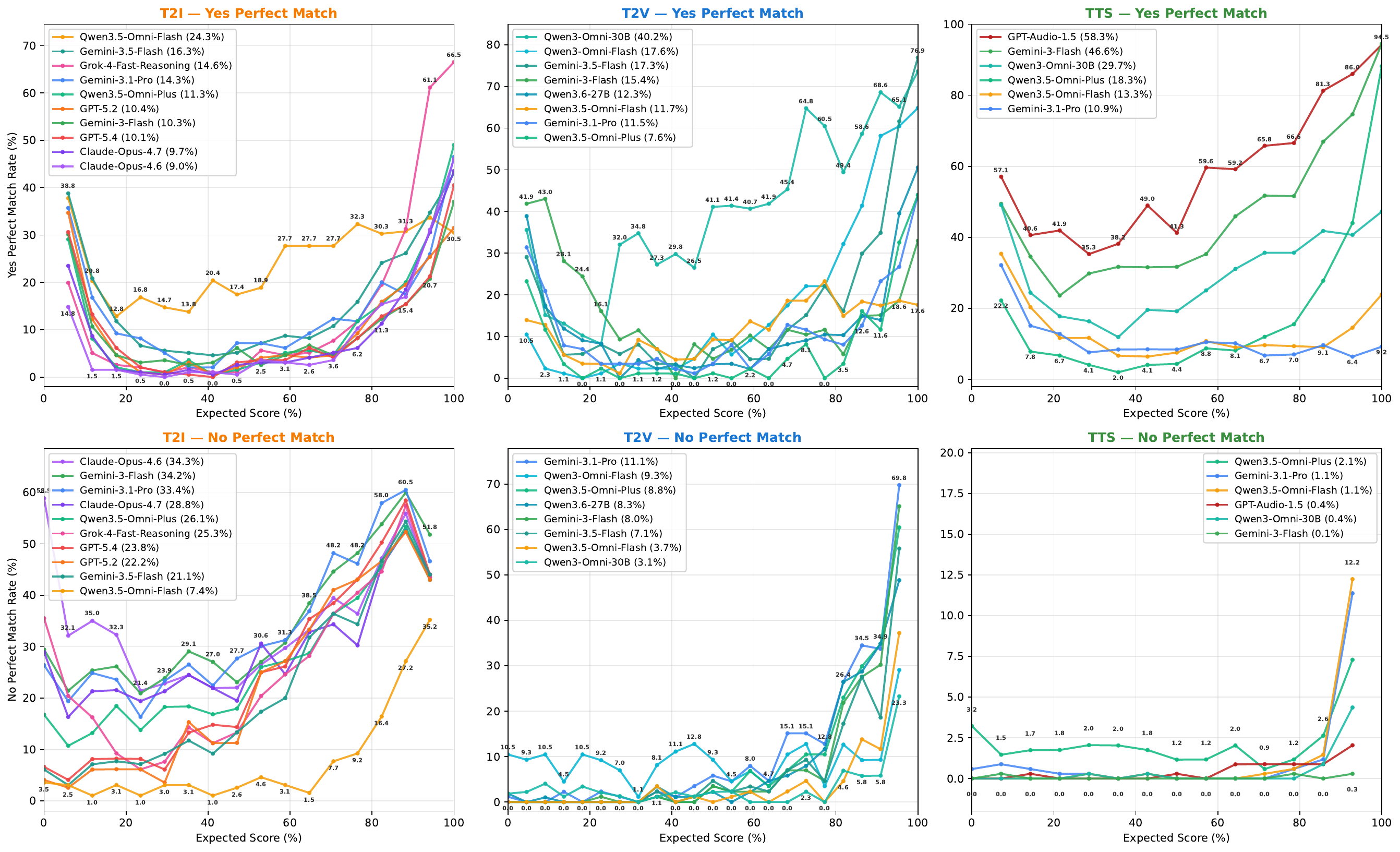}
\caption{Per-segment Yes/No perfect-match rates. \textbf{Top row:} Yes
perfect-match, correct on all ground-truth-Yes dimensions.
\textbf{Bottom row:} No perfect-match, the symmetric quantity on
ground-truth-No dimensions. A calibrated judge would show roughly
symmetric rows; the visible asymmetry is class bias, and the low
No-row values quantify how poorly the few defects are localized.}
\label{fig:perfect}
\end{figure*}

\paragraph{Judges confirm far better than they detect.}
On each task the peak Yes-rate far exceeds the peak No-rate and no judge is symmetric, and the gap widens sharply
as the modality moves from vision to audio: the peak No perfect-match
is $60.51\%$ on T2I (Gemini-3.1-Pro), $69.77\%$ on T2V (Gemini-3.1-Pro)
and only $12.24\%$ on TTS (Qwen-3.5-Omni-Flash), while the Yes side
stays high and reaches $94.52\%$ on TTS (Gemini-3-Flash), whose
overall TTS Yes rate ($46.7\%$) is more than four times that of its
sibling Gemini-3.1-Pro ($10.9\%$), a within-family split that echoes
the TTS family collapse of Sec.~\ref{subsec:setup}. The polarity
of the imbalance is itself a personality: on T2I a lenient judge such
as Qwen-3.5-Omni-Flash carries the highest Yes but the lowest No
perfect-match, whereas strict judges such as Claude-Opus-4.6 and
Gemini-3-Flash reach a high No ($\sim\!34\%$) but a low Yes, so class
bias does not even track family membership (on T2I Qwen-3.5-Omni-Flash
is Yes-biased, with Yes-rate minus No-rate equal to $+16.9$
while its sibling Qwen-3.5-Omni-Plus is No-biased at $-14.8$). On TTS
the asymmetry is extreme: GPT-Audio-1.5 confirms $58.3\%$ of
all-satisfied speech yet catches all-violated dimensions on only
$0.44\%$ of samples, and across the six TTS judges the No perfect-match
averages between $0.06\%$ and $2.15\%$ (the re-run Qwen-3-Omni-30B no
exception at $0.42\%$).

\paragraph{The high-score rebound is confirmation, not localization.}
This is the key complement to Figure~\ref{fig:multi_total}: the
high-score rebound in segment accuracy does not mean a judge has found
the defects. Restricting to high-score segments (at least $76\%$ of the
requirements satisfied, so only one or two No dimensions remain), the
mean No perfect-match is $44.92\%$ on T2I, $23.31\%$ on T2V and merely
$2.67\%$ on TTS (a peak of $12.24\%$, an all-segment mean below $1\%$).
In other words, a TTS judge that posts a strong mid-to-high segment
average is, on the very samples where only one or two defects remain,
almost never able to name them: its high accuracy is carried entirely
by confirming satisfied dimensions. The ordering $44.9\%$ on the visual
task, $23.3\%$ on video, $2.7\%$ on speech, is a modality statement:
when a judge can no longer rely on visual evidence and has to attribute a defect from temporal or audio cues alone, its localization ability
collapses. The capability center of today's OmniJudges is still
visual.

\paragraph{At the all-violated end the gap is combinatorial, and total
on temporal tasks.}
The low-score end makes the mechanism explicit. When the ground-truth
total is zero every dimension is a defect, so No perfect-match asks the
judge to catch all of them at once. Per-dimension accuracy can still look
reassuring here, yet the all-or-nothing rate is far lower: GPT-5.2 reads
$84.9\%$ per dimension on all-violated T2I samples but scores only
$4.06\%$ of them fully correct, the mark of an all-or-nothing score that
compounds many per-dimension decisions. On T2I the stricter judges still
retain a usable rate (Claude-Opus-4.6 $58.9\%$, Grok-4-Fast $35.5\%$),
but on the temporal tasks it collapses to the floor for every judge: on
T2V seven of the eight stay at or below $1.9\%$ (four at exactly $0\%$,
the maximum being Qwen-3-Omni-Flash at $10.5\%$), and on TTS five of the
six stay at or below $0.6\%$ (peak Qwen-3.5-Omni-Plus $3.2\%$). Two
forces compound. The metric runs over the maximum number of No
dimensions, so the Yes-bias every judge carries on low-score segments
(Sec.~\ref{subsec:segment}) is almost certain to leak at least one false
Yes and void the sample; and a temporal defect cannot be read from a
single frame but must be attributed from evidence spread over time or
audio, exactly the per-dimension detection that is weakest there and
that the coupling map shows is frequently decided by overall impression
rather than dimension by dimension (Sec.~\ref{subsec:coupling}). On T2V
and TTS a judge therefore almost never labels a wholly defective sample
as defective throughout.

\paragraph{The asymmetry traces back to the generation-judging
pipeline.}
The Yes-bias reproduces across every judge, both modalities and exactly
the segments where the test set forces a No commitment. Real-world
generations skew toward acceptable samples; if pre-training and
benchmark suites inherit that skew, an all-Yes default is locally
optimal almost everywhere. Without low-score samples in a
dimension-and-score-balanced benchmark, this entire diagnosis would be
unreachable: there would be no No perfect-match to fall toward zero in
the first place.

\paragraph{Cross-modal synthesis.}
On T2I and T2V the two rows roughly mirror each other across judges, so
a Yes-biased judge is paired with a No-biased counterpart and a balanced
evaluator can be assembled from the pool. On TTS the rows are
incommensurable: the Yes row approaches perfect prediction while the No
row hugs the floor for every judge. The Yes/No gap thus identifies TTS
as the most extreme manifestation of the generation-judging asymmetry,
and pinpoints class-prior correction, rather than fine-grained
perception alone, as the primary data-side intervention the TTS
sub-leaderboard demands.

\subsection{Judgment Decoupling: Pseudo-Decoupling Diagnosis}
\label{subsec:coupling}

\paragraph{Reading the matrix.}
A faithful OmniJudge should decide each near-orthogonal requirement on
its own, so we compute the off-diagonal Pearson correlation matrix of
each judge's per-dimension predictions (Figure~\ref{fig:corr}). The
reading is a statement about the judge, not about the dimensions: a
high correlation means the judge lets its decision on one dimension
drag its decision on another (a halo effect), whereas a low correlation
means it decides the two independently. In the figure a coloured pie
marks a pair on which one or more judges exceed $\rho\!>\!0.6$ (each
slice is one model) and a white cell reports the cross-judge mean
correlation when none does; counts below are over unordered dimension
pairs ($136$ on T2I, $231$ on T2V, $91$ on TTS). One caveat applies to the white cells: a white, low-correlation cell is only evidence of
independent judgment when the dimension is judged well above chance,
since a near-chance dimension such as T2V Video Reality ($0.42$--$0.48$,
Sec.~\ref{subsec:per_dim}) produces low correlations by near-constant
prediction rather than by genuine independence.

\paragraph{Pseudo-decoupling, defined.}
We term a judge \emph{pseudo-decoupled} when its macro accuracy remains competitive while its predicted dimension-vs-dimension matrix is dominated by a few high-$\rho$ blocks,
indicating that nominally distinct decisions are internally driven by a
single latent decision; we treat $\rho\!>\!0.8$ as strong entanglement.

\begin{figure*}[!t]
\centering
\includegraphics[width=\textwidth]{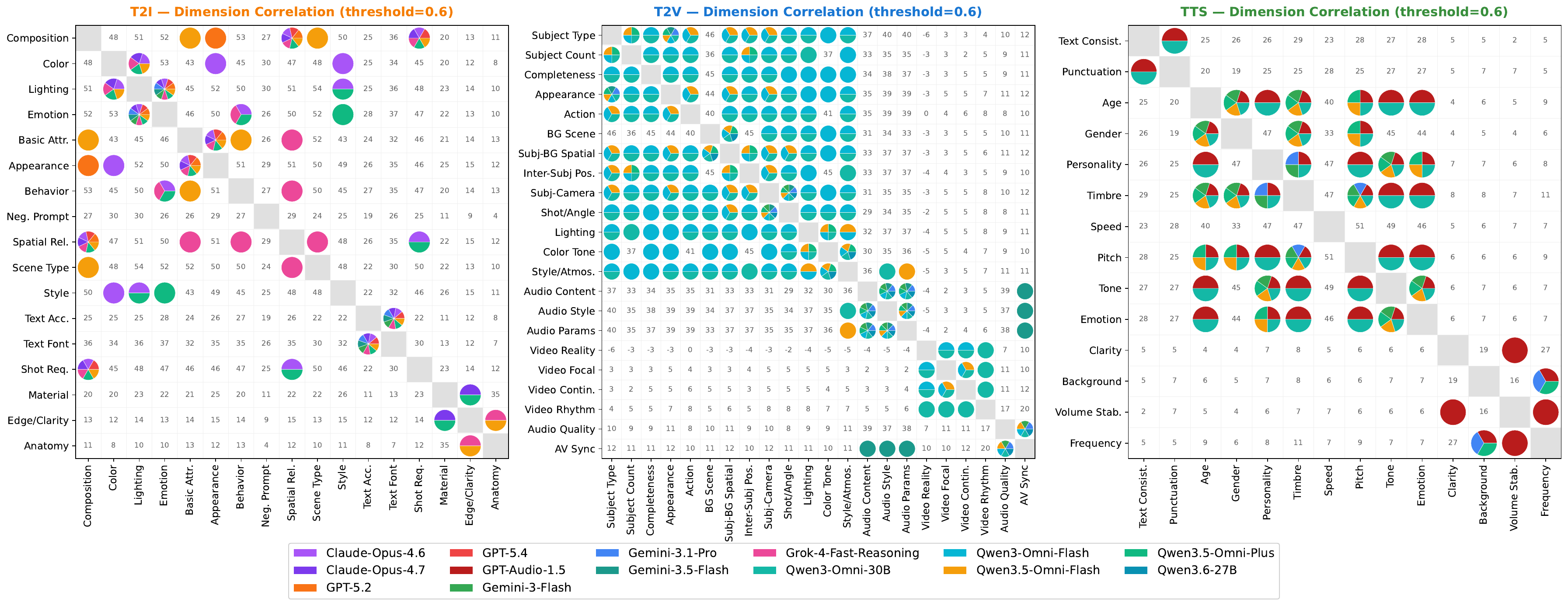}
\caption{Predicted dimensional correlation. Each cell is a dimension
pair: a coloured pie indicates one or more judges exceed $\rho\!=\!0.6$
on that pair (each slice is one model); a white cell shows the mean
correlation ($\times100$) when no judge exceeds the threshold. A
coloured cell reflects the judge coupling its own decisions, not a
property of the dimension design.}
\label{fig:corr}
\end{figure*}

\paragraph{The entangled pairs are semantically adjacent, and shared.}
Whatever the model, the high-$\rho$ pairs land on semantically adjacent
dimensions rather than at random: text-accuracy with text-typesetting
on T2I, the audio-content trio on T2V, and the speaker-identity
attributes (pitch, timbre, age, gender) on TTS. This is a shared
capability limit rather than an idiosyncrasy of one judge. Which pairs
a judge fails to separate, however, is model-specific: T2I
colour-with-lighting is coloured for Claude, Grok and Qwen but stays
white for GPT and Gemini, so the same pair is entangled or not
depending on the judge.

\paragraph{Entanglement density varies sharply across judges.}
On T2V the entangled-pair count spans more than an order of magnitude within the same task, from $5$--$8$ for the most decoupled judges to $67$ and $74$ for
the two most entangled (a more-than-tenfold gap, $16$--$18$ of the
latter strong). Notably, this range runs \emph{within} a family and tracks version
iteration: entanglement falls steeply across Qwen generations, from
$67$--$74$ pairs for Qwen-3-Omni to $11$--$23$ for Qwen-3.5-Omni and
only $7$ for Qwen-3.6-27B (on par with any Gemini judge), so decoupling
is a judge-level capability that newer releases visibly acquire rather
than a property of the vendor.
The most entangled judges compress much of the prompt-related half of
T2V into a handful of meta-judgments while still reporting
$\sim\!70\%$ macro accuracy, the pseudo-decoupling signature.
On TTS the most entangled judges are GPT-Audio-1.5 ($22$ pairs, $7$
strong) and the re-run Qwen-3-Omni-30B ($19$ pairs, $5$ strong), and
GPT-Audio-1.5 is also the macro-lowest TTS judge; the two Gemini TTS
judges, by contrast, stay near-orthogonal ($3$ and $7$ pairs) yet only
reach the middle of the TTS table, so low coupling is necessary but not
sufficient for a high rank, and the diagnosis should be read together with the per-dimension competence of Sec.~\ref{subsec:per_dim}.
On T2I the coupling is mild for all judges ($2$--$11$ pairs), with the
Gemini family the most decoupled ($2$ each) and Claude, Grok and
Qwen-3.5-Omni-Plus the most entangled ($11$ each).

\begin{table}[H]
\centering
\small
\begin{tabular}{l r r r}
\toprule
TTS judge & Entangled pairs ($\rho\!>\!0.6$) & Macro acc.\ & Macro rank \\
\midrule
Qwen-3.5-Omni-Plus  & $10$ & $65.70\%$ & 1 \\
Gemini-3.1-Pro      & $\phantom{0}3$ & $60.70\%$ & 2 \\
Qwen-3.5-Omni-Flash & $\phantom{0}9$ & $60.44\%$ & 3 \\
Gemini-3-Flash      & $\phantom{0}7$ & $60.18\%$ & 4 \\
Qwen-3-Omni-30B     & $19$ & $60.12\%$ & 5 \\
GPT-Audio-1.5       & $22$ & $58.71\%$ & 6 \\
\bottomrule
\end{tabular}
\caption{Entangled-pair count vs.\ macro accuracy on TTS. The most
entangled judge (GPT-Audio-1.5) is also the weakest, but the reverse
does not hold: Gemini-3.1-Pro is the most decoupled yet only mid-pack,
because it is separately floored on the audio-quality dimensions of
Sec.~\ref{subsec:per_dim}.}
\label{tab:coupling_macro}
\end{table}

\paragraph{Cross-modal synthesis.}
Each entangled cluster collapses around a single perceptual cue that
the judge has substituted for the decoupled decisions: legible text on
T2I, ``audio matches the prompt'' on T2V, and voice identity on TTS.
Whenever a macro number moves without a matching change in the coupling
map, the improvement is structural rather than capability-driven.

\begin{figure*}[!t]
\centering
\includegraphics[width=\textwidth]{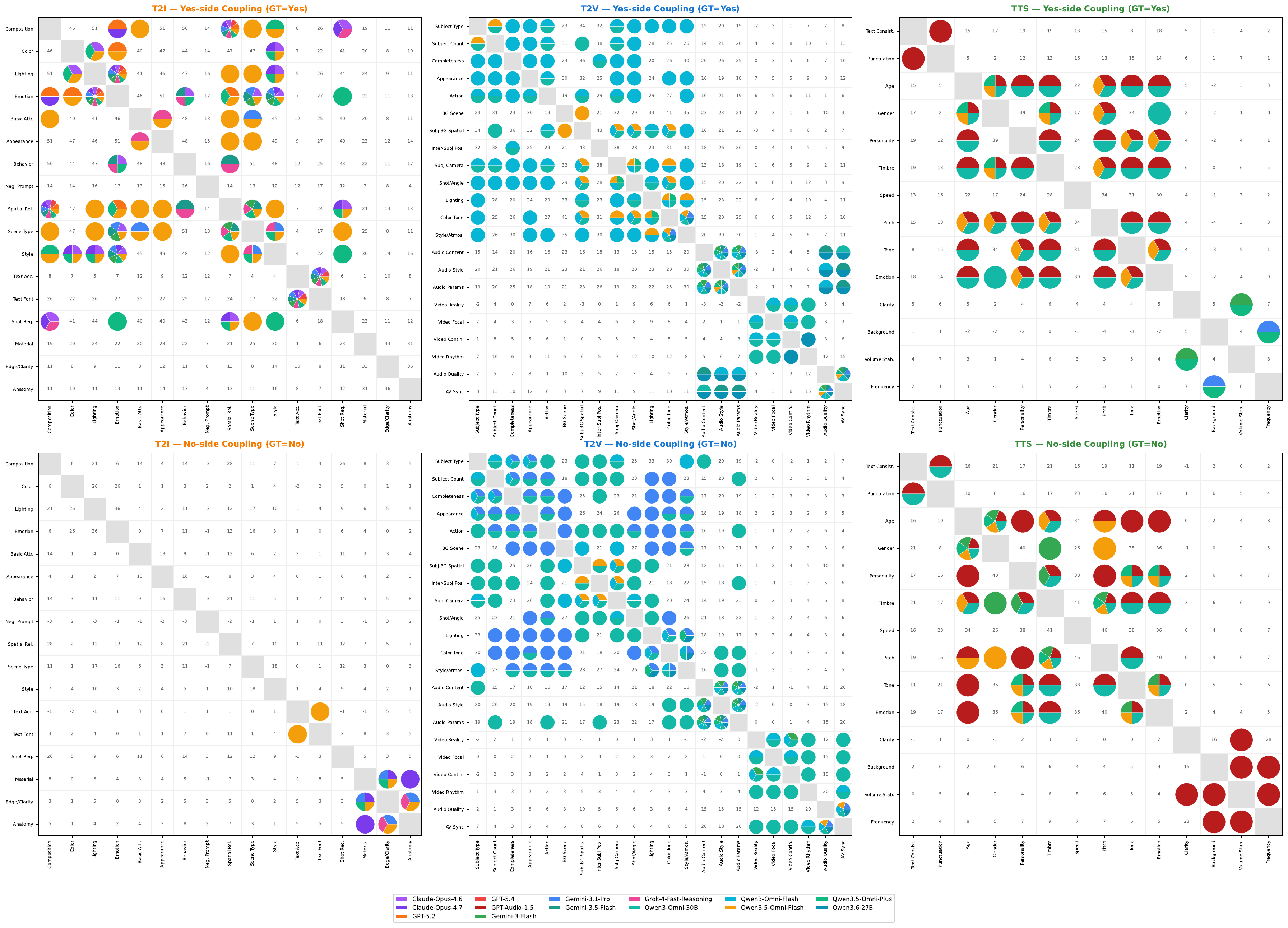}
\caption{Polarity-conditioned dimensional coupling. \textbf{Top row:}
correlation computed only on ground-truth-Yes dimension pairs
(confirmation-side halo); \textbf{bottom row:} only on ground-truth-No
pairs (defect-side halo). Glyph and threshold follow
Figure~\ref{fig:corr}; cells with insufficient same-polarity samples or
a near-constant prediction are left undetermined.}
\label{fig:corr_yesno}
\end{figure*}

\paragraph{Which polarity carries the coupling.}
Splitting this matrix by ground-truth polarity
(Figure~\ref{fig:corr_yesno}) shows that
the coupling is not symmetric across the Yes and No decisions. On T2I it
is almost entirely a confirmation-side halo: nearly every coupled cell
sits on the Yes side, and the only defect-side coupling is confined to
the fine-grained modality block (Material, Edge \& Clarity, Object
Anatomy). On T2V both sides couple heavily, and on TTS the two are
balanced. The split also exposes couplings that the mixed matrix
averages away, and judges that look independent overall yet couple on a
single polarity (for example Gemini-3-Flash on TTS Audio Clarity with
Volume Stability, $0.25$ in the mixed matrix but $0.71$ on the Yes
side). The most instructive case is Gemini-3.1-Pro on T2V: the strongest
and most balanced video judge overall, it nonetheless couples almost
exclusively on the defect side (Completeness with Subject Count rising
from $0.54$ in the mixed matrix to $0.73$ on the No side), fusing dimensions specifically when attributing a defect. The mixed
matrix of Figure~\ref{fig:corr} therefore understates the coupling.

\subsection{Per-Dimension Competence: the Chance Floor and Capability Gaps}
\label{subsec:per_dim}

\paragraph{Reading the fan.}
Figure~\ref{fig:dim} is a $53$-axis tri-sector radar that reads like a
three-bladed fan: each blade is one modality (T2I $17$ dims, T2V $22$,
TTS $14$), and within a blade the dimensions are sorted by best-model
accuracy clockwise from 12 o'clock, so travelling inward along a blade
traces each judge's capability gradient from its strongest to its
weakest dimension. The radius is per-dimension judging accuracy (inner
$30\%$ to outer $100\%$), so the outer arc collects the dimensions a
judge decides reliably and the collapsed inner region the shared blind
spots, while the spread between the four contours on one axis is a
per-dimension capability gap between judges. We overlay only the four
judges common to all three tasks (Gemini-3-Flash, Gemini-3.1-Pro,
Qwen-3.5-Omni-Flash, Qwen-3.5-Omni-Plus) so the comparison is
cross-modal, and the English and Chinese dials are drawn side by side so
a judge's cross-lingual stability can be read off directly.

\begin{figure*}[!t]
\centering
\includegraphics[width=0.85\textwidth]{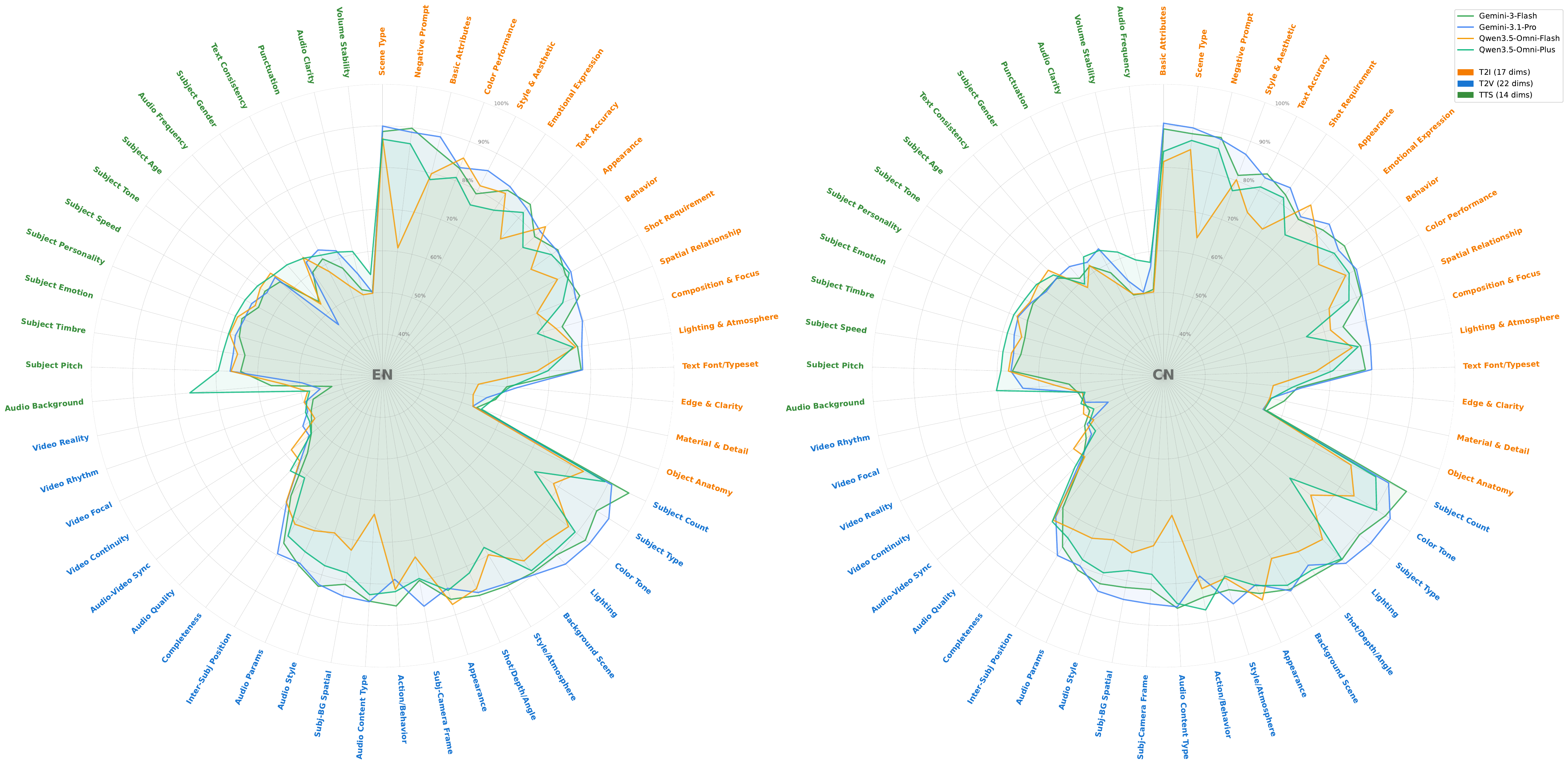}
\caption{$53$-dimension tri-sector radar for T2I ($17$ dims), T2V
($22$ dims) and TTS ($14$ dims), EN on the left and CN on the right.
Four cross-task judges are overlaid; dimensions are sorted by per-dim
best accuracy (strongest at 12 o'clock). The inner region collapsing
toward the $50\%$ ring marks shared blind spots; the spread between
contours on an axis marks a capability gap.}
\label{fig:dim}
\end{figure*}

\paragraph{Strengths are semantic; blind spots are fine-grained and
physical.}
Every sector is strongest on a semantic or categorical dimension
(T2I Scene Type $0.899$, T2V Subject Count $0.955$, TTS Audio
Background $0.765$) and weakest on a perceptual, fine-grained one, with
T2V showing the widest internal range of the three sectors, from
Subject Count ($0.955$) at the top down to its physical dimensions at
the floor. All four judges trace nearly the same blade profile, full at
the rim and pinched at the hub, so the strong-to-weak dimension ordering
is shared rather than model-specific; and the whole TTS blade retracts
inward (radii $0.60$--$0.66$) relative to the comparable T2I and T2V
blades ($0.71$--$0.79$), marking speech as the hardest modality to
judge. The single deepest blind spot is T2V Video Reality, on which the
best of the four judges reaches only $0.484$, below chance, with Video Rhythm
($0.498$) and Video Focal ($0.505$) close behind; the T2I floor is
Object Anatomy ($0.550$) and Material \& Detail ($0.577$), and the TTS
floor is Volume Stability ($0.545$) and Audio Clarity ($0.607$).
Several of these lock onto values indistinguishable from a constant
prediction, for example Qwen-3.5-Omni-Flash at $0.5001$ on TTS Audio
Clarity, the textbook signature of a dimension the judge has not
learned to evaluate.

\paragraph{The chance floor has two compatible readings, and a caveat.}
A near-$50\%$ score on a dimension can arise either because the output
is locked to an entangled companion dimension
(Sec.~\ref{subsec:coupling}) or because the judge emits a constant
token; both converge on the same conclusion, namely that the
fine-grained capability is not activated. This also qualifies the
decoupling diagnosis: a dimension can appear decoupled simply because
the judge emits a near-constant prediction on it, so a low correlation
combined with a
near-chance accuracy is spurious independence rather than genuine
independent judgment, and the two figures should be read together. A
per-polarity accuracy split
(Figure~\ref{fig:dim_yesno}) further shows the floor to be specifically a
near-constant Yes prediction rather than unbiased guessing: on
essentially every near-chance dimension the Yes accuracy is near $1$
while the No accuracy is near $0$ (T2V Video Focal $0.99$/$0.02$, TTS
Volume Stability $0.99$/$0.03$), so on these dimensions the judge is
closer to a rubber stamp than to a detector. A judge that lacks an input modality
altogether falls onto the same floor: Qwen-3.6-27B, which has no audio
channel, sits near chance on all five of its T2V audio dimensions
($0.46$--$0.58$), splitting into a constant-Yes default on the
intrinsic quality and sync dimensions and near-chance responses on the
prompt-alignment ones (Appendix~\ref{tab:qwen36_audio}).

\paragraph{The English and Chinese dials nearly coincide.}
Because the fan is drawn once per language, a judge's cross-lingual
stability is legible as the near-coincidence of its two dials: the
English and Chinese blades overlap almost exactly, the strong-to-weak
ordering along each blade is preserved (Video Reality stays pinned at
the T2V hub in both), and the residual difference is a mild
Chinese-favoring shift confined to a few audio dimensions rather than a
global rotation. Qwen-3.5-Omni-Flash is the most language-stable of the
four judges and Gemini-3.1-Pro the most sensitive;
Sec.~\ref{subsec:crosslingual} quantifies both across all four diagnoses.

\begin{figure*}[!t]
\centering
\includegraphics[width=\textwidth]{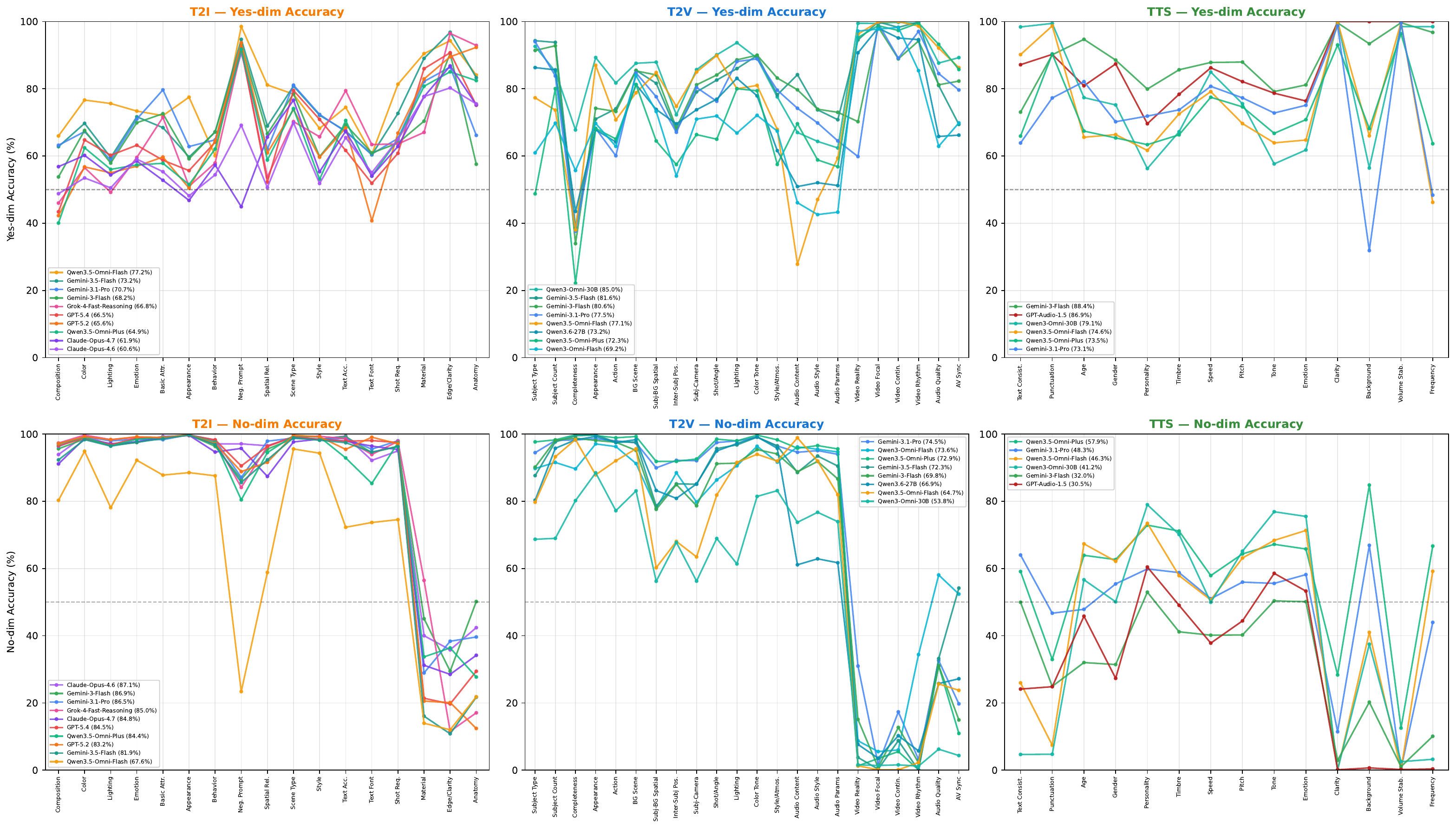}
\caption{Per-dimension Yes/No accuracy for every judge on T2I/T2V/TTS.
\textbf{Top row:} Yes accuracy (recall on ground-truth-Yes dimensions);
\textbf{bottom row:} No accuracy (specificity on ground-truth-No
dimensions). A dimension whose Yes row is high while its No row is near
zero is a near-constant ``Yes'' predictor, which averages to the
near-chance radius seen in Figure~\ref{fig:dim}.}
\label{fig:dim_yesno}
\end{figure*}

\paragraph{Aggregating over $\mathcal{D}_p$ versus $\mathcal{D}_m$.}
Averaging the per-dimension accuracies over the prompt-related and
modality-related halves defined in Sec.~\ref{sec:dimensions}
(T2I $\mathcal{D}_m\!=\!\{$D15--D17$\}$,
T2V $\mathcal{D}_m\!=\!\{$D17--D22$\}$,
TTS $\mathcal{D}_m\!=\!\{$D11--D14$\}$) is reported in the
$\mathcal{D}_p$/$\mathcal{D}_m$ columns of Table~\ref{tab:dp_dm_macro}. On T2I and T2V every judge scores
$18$--$35$ points lower on $\mathcal{D}_m$ than on $\mathcal{D}_p$,
with the largest drop reached by Gemini-3-Flash on T2V
($84.81\%$ vs.\ $49.52\%$, $-35.29$). On TTS the picture inverts the
old expectation: the modality gap is smallest not for the Gemini judges
but for Qwen-3.5-Omni-Plus ($66.33\%$ vs.\ $64.20\%$, only $-2.12$),
the one judge that has genuinely learned the audio-quality block, while
both Gemini judges keep a normal $-10$ to $-15$ point gap. The Yes/No
split (Figure~\ref{fig:dim_yesno}) shows this reflects a genuine ability rather than
a fortunate label prior: on the audio-quality dimensions the other judges
collapse to a near-constant ``clean'' verdict (No-accuracy $\approx\!0$) and rarely flag degraded audio, whereas
Qwen-3.5-Omni-Plus alone stays two-sided where the signal permits
(No-accuracy $0.85$ and $0.67$ on two of the four), genuinely
separating clean from degraded speech. The macro
average collapses these qualitatively different
failure modes into one number: without dimension parity, the long tail
of easy prompt dimensions would dilute the $\mathcal{D}_m$ floor and a
$20$--$35$ point cliff would look like a few-point dip.

\paragraph{Cross-modal synthesis, and a reading caveat.}
The radar reduces the per-dimension diagnosis to one rule: T2I and T2V
share a bottom-region floor confined to $\mathcal{D}_m$ and large
top-region spreads on $\mathcal{D}_p$, so improving these judges is
mostly a matter of activating dormant prompt-side capabilities. A
radius on this chart, however, is a Yes/No mixed average and cannot by
itself tell whether a high value reflects genuine detection or mere
confirmation of the many satisfied dimensions; that distinction
requires the perfect-match view of Sec.~\ref{subsec:yesno}. Read this
way, the moderate TTS radii ($0.60$--$0.66$) are especially deceptive.

\subsection{Cross-Lingual Robustness (English vs.\ Chinese)}
\label{subsec:crosslingual}

All four diagnoses were run in English and Chinese; only the English
figures are shown in the main text, and the Chinese-language
counterparts of Figures~\ref{fig:multi_total}--\ref{fig:dim} are
deferred to the appendix. The picture is language-invariant. The macro
leader is the same in both languages on every task (Gemini-3.1-Pro on
T2I, Gemini-3.5-Flash on T2V, Qwen-3.5-Omni-Plus on TTS), and per-model
accuracy barely moves across languages: on the $53$-dimension radar the
English and Chinese means of any judge differ by at most $0.022$ (the
largest being Gemini-3.1-Pro on TTS, $0.607$ vs.\ $0.629$), with Chinese
marginally higher on the visual tasks. Every structural conclusion
survives the switch: Video Reality remains the deepest blind spot in
both languages ($0.484$ in English, $0.513$ in Chinese); the No-class
collapse on TTS persists (high-score No perfect-match $2.67\%$ in
English vs.\ $1.64\%$ in Chinese, against $44.9\%$ and $45.3\%$ on T2I);
and the most entangled video judges remain the most entangled in both
languages ($74$ then $65$ pairs), with only marginal count changes.

The residual language sensitivity is concentrated on a handful of
fine-grained audio and temporal dimensions rather than spread across the
taxonomy. The largest per-dimension swings are all Chinese-favoring and
sit on TTS Audio Background (Gemini-3.1-Pro, $0.494$ to $0.639$), TTS
Audio Frequency (Gemini-3.1-Pro, $0.462$ to $0.564$) and T2V
Completeness (Qwen-3.5-Omni-Plus, $0.609$ to $0.741$); the Chinese
matrices are also marginally more entangled on the same
semantically-adjacent pairs, without changing the decoupling structure.
D\textsuperscript{3}-Omni thus yields the same capability diagnosis in
both languages, and the few cross-lingual gaps are themselves
diagnostic, isolating a small set of speaker-timbre and audio-frequency
dimensions on which a judge's competence is language-dependent.

\subsection{Synthesis: From Diagnosis to Construction}
\label{subsec:synthesis}

The four diagnostic axes converge on a consistent picture. We summarize
four shortcomings and, for each, indicate how it should steer the
construction of training and evaluation data (Sec.~\ref{sec:method}),
so that any lab can re-aim the same pipeline at the exact gap diagnosed
in its own model.

\paragraph{The failure modes are industry-wide, not vendor-specific.}
Before mapping the shortcomings, we stress that these behaviors
recur across every judge and, crucially, across vendors. The U-shaped
competence collapse on the mixed-quality middle, the regression-to-mean calibration bias that over-scores low-quality content and under-scores high-quality content
(Sec.~\ref{subsec:segment}), and the temporal-endpoint asymmetry by
which a judge confirms all-good content far more readily than it catches
all-bad content on T2V and TTS, all hold for the full panel rather than
for a single family; the temporal-endpoint asymmetry in particular is
positive without exception, for all eight T2V and all five TTS judges in
both languages. The clearest sign that these are properties of
current OmniJudges as a class is cross-family convergence: on the T2V
endpoint asymmetry the Gemini and Qwen families land on nearly identical
means ($+19.30$ and $+18.91$), so the confirm-heavy temporal profile is
a shared regularity of today's judges rather than one vendor's
implementation artifact. This is what makes the four shortcomings below
worth closing generically rather than model by model.

\paragraph{One thread across the four figures.}
The four diagnostics share a single reading protocol. The
segment-average accuracy of Figure~\ref{fig:multi_total} and the
per-dimension radii of Figure~\ref{fig:dim} are both means over
dimensions, so they are inflated whenever a judge merely confirms the
many satisfied dimensions of a high-score sample. Whether the few
remaining defects are actually caught, and whether each dimension is
decided on its own, is revealed only by the No perfect-match of
Figure~\ref{fig:perfect} and the decoupling map of
Figure~\ref{fig:corr}. Read together, the four figures separate a
genuine capability from a confirmation artifact, and the gap between the
two widens from vision to speech: on high-score samples the No
perfect-match falls from $44.9\%$ on T2I to $23.3\%$ on T2V and $2.7\%$
on TTS, so an aggregate score can stay high precisely where defect
attribution has collapsed. A per-polarity decomposition
(Figures~\ref{fig:corr_yesno} and~\ref{fig:dim_yesno}) sharpens both halves of this thread: the
mean-inflation is specifically a near-constant Yes prediction rather
than random guessing, and the entanglement splits into a
confirmation-side halo (dominant on T2I) and a defect-side halo
(dominant on the temporal tasks), so even a judge that looks decoupled
overall can fuse dimensions on one polarity alone.

\paragraph{(S1) Modality-internal perception is the floor of every
judge.}
The modality-related dimensions cluster near chance on the inner ring
of Figure~\ref{fig:dim}, with several locking onto constant-prediction
values just above $50\%$ (Sec.~\ref{subsec:per_dim}), the deepest being
T2V Video Reality at $0.484$.

\paragraph{(S2) Decisions collapse along modality-specific clusters.}
Despite competitive macro accuracy, judges fuse nominally orthogonal
dimensions into a single bit, and the cluster is task-specific: a text
pair on T2I, an audio trio on T2V, a speaker block on TTS. Wherever a
strong macro number coexists with a dense coloured block in
Figure~\ref{fig:corr}, it is only pseudo-decoupled
(Sec.~\ref{subsec:coupling}); on the temporal tasks the most entangled
judges accumulate dozens of such pairs.

\paragraph{(S3) Class priors dominate, and aggregate accuracy masks poor
attribution.}
The U-shaped curves and the TTS Yes-bias both stem from majority-class
shortcuts; more sharply, the high-score rebound in segment accuracy is
not defect localization, since the No perfect-match falls from $44.9\%$
on T2I to $2.7\%$ on TTS on the same high-score segments
(Sec.~\ref{subsec:segment}, Sec.~\ref{subsec:yesno}).

\paragraph{(S4) No OmniJudge is universal.}
The leaderboard reorders between vision and speech: the Gemini family
leads T2I and T2V but Qwen-3.5-Omni-Plus leads TTS and is the only judge
to close the audio-quality gap. Cross-modal generalization of judging
ability remains open (Sec.~\ref{subsec:setup}).

\paragraph{How the diagnoses steer data construction.}
Read as requirements on the data rather than as a scoreboard, the four
shortcomings point to concrete adjustments in how training and
evaluation triplets are built. (S1) The modality floor calls for
over-representing negatives that perturb only the modality-related
dimensions $\mathcal{D}_m$ while holding the prompt fixed, so the data
rewards genuine perceptual discrimination instead of prompt-level shortcuts.
(S2) The entangled clusters call for mining triplets that violate
exactly one member of a highly correlated pair, forcing each dimension to be
decided on its own and breaking the halo that inflates macro accuracy.
(S3) The class-prior shortcuts call for score-uniform sampling---equal
mass at every total score $s\!\in\!\{0,\dots,D\}$ and a balanced
Yes/No count per dimension---with extra all-violated, low-score
negatives on the temporal tasks, where defect attribution collapses.
(S4) The cross-task reordering calls for a joint, task-agnostic
balanced schema spanning T2I/T2V/TTS, so that cross-modal judging is
trained for rather than assumed. These are precisely the decoupled,
dual-balanced, and dynamic construction moves of
Sec.~\ref{sec:method}, now aimed at the specific cells each model
fails.

\paragraph{Closing the loop.}
We intend these diagnoses as a constructive starting point for
improvement rather than a final assessment of any model. The
same dynamic, dual-balanced, decoupled construction can be re-aimed at
whichever cells of Figures~\ref{fig:multi_total}--\ref{fig:dim} a given
judge is weakest on, turning a diagnostic benchmark into a broad recipe
for supplementary training data. We also state a limitation plainly:
the seed corpus and the pool of negative-construction operators used
here are our own and far from exhaustive, so the coverage of any single
effort---including ours---is necessarily partial. We therefore
encourage the community to plug in its own seed pools and its own
libraries of negative-construction operators, broadening and enriching
the balanced benchmark so that the same method surfaces a more
complete picture of each model's blind spots. Pooled this way, the
individual diagnoses can grow into a genuinely collective multimodal
diagnostic lens that keeps pace with Omni-LLMs as they evolve.

%% file: conclusion.tex
\section{Conclusion}
\label{sec:conclusion}

We introduced \textbf{D\textsuperscript{3}-Omni}, the first balanced
and decoupled benchmark for diagnosing fine-grained multimodal
understanding in present-day OmniJudges. Built through the
\textbf{D\textsuperscript{3}} construction framework
(Sec.~\ref{sec:method}), the benchmark combines
\textbf{Dual-balanced} sampling that closes the negative-sample gap
and enforces near $1\!:\!1$ Yes/No parity on every dimension,
\textbf{Decoupled} orthogonal-flip operators that fix verified
positive seeds and perturb one dimension at a time, and a
\textbf{Dynamic} construction loop that steers generation toward the
sparsest regions of the label space. The released benchmark spans
T2I, T2V and TTS with $53$ near-orthogonal binary dimensions
($17$/$22$/$14$) over $10{,}671$ samples and a uniform distribution
across all total-score levels, while remaining a stable evaluation
set.

Applied to ten T2I, eight T2V and six TTS judges, the same balanced
lens converts an apparently saturated leaderboard
into four mutually reinforcing diagnoses
(Sec.~\ref{subsec:segment}--\ref{subsec:yesno}, synthesized in
Sec.~\ref{subsec:synthesis}):
\textbf{(S1) modality-internal perception is the universal floor},
with $\mathcal{D}_m$ macro accuracies floored near $50\%$ on every task
and an exact $0.500x$ constant-prediction signature on TTS Audio
Clarity;
\textbf{(S2) decisions collapse along modality-specific clusters},
with the most entangled video judges compressing much of the prompt
half into a few meta-judgments---dozens of entangled T2V dimension
pairs---while still reporting $\sim\!70\%$ macro accuracy;
\textbf{(S3) class priors dominate when the sample regime changes},
producing U-shaped or monotonic accuracy curves under which every judge
confirms satisfied content far more reliably than it detects violated
content, leaving the No
class almost unlabeled on TTS (No perfect-match $\leq\!12.24\%$ across
all six judges, and only $2.7\%$ even on high-score segments, versus
$44.9\%$ on T2I); and \textbf{(S4) no OmniJudge generalizes across
modalities}, with the TTS podium led by a different family than the
visual tasks (Gemini-3.1-Pro $78.56\%$ on T2I but $60.70\%$ on TTS,
while Qwen-3.5-Omni-Plus leads TTS at $65.70\%$). Crucially, all four
recur across model families rather than isolating a single vendor, so
they characterize current OmniJudges as a class. Together, these results
identify the gap between \textbf{OmniJudge} and \textbf{OmniBias}:
aggregate accuracy can be high while the underlying capability
distribution is shallow, entangled and class-prior-driven, and only a
benchmark balanced on both dimensions and total scores can surface
that gap.

The same construction pipeline is the natural prescription for the
diagnosis it surfaces. The orthogonal-flip operator targets S1 by
generating one-hot-$\mathcal{D}_m$ negatives that share the same
prompt as a positive seed, the attainability guarantee
(Lemma~\ref{thm:existence}) targets S2 by mining triplets that violate
exactly one member of an entangled pair, the score-uniformity
constraint targets S3 by removing the empirical reward for all-Yes /
all-No shortcuts, and the task-agnostic schema targets S4 by enabling
joint balanced fine-tuning across T2I, T2V and TTS without bespoke
per-modality pre-processing. D\textsuperscript{3}-Omni therefore
delivers more than a leaderboard: it is a closed-loop instrument that
turns each diagnosed weakness into an actionable training-data
prescription, and, since our own seed corpus and pool of
negative-construction operators are inevitably partial, we invite the
community to plug in its own seed pools and operator libraries,
generate dimension- and score-balanced triplets aimed at each model's
weakest cells of Figures~\ref{fig:multi_total}--\ref{fig:perfect}, and
grow these individual diagnoses into a genuinely collective multimodal
diagnostic lens.

\paragraph{Limitations and future work.}
Three directions complement the present study. First, the current
benchmark covers T2I, T2V and TTS; extending the same balanced
decoupled construction to image-to-X, video-to-X and any-to-text
judging would test whether the four shortcomings persist in the
reverse direction. Second, sample-level explainability that pairs
each diagnostic curve with a representative example would make the
diagnoses inspectable without re-running every judge; scaling this to
a public gallery is a natural next step. Third, the dynamic loop is currently triggered
by sparsity in the label distribution; coupling it directly to a
target judge's most recent error profile would turn
D\textsuperscript{3}-Omni from a stable evaluation set into an
adaptive training-data generator, closing the diagnose-to-construct
loop in real time.

%% file: appendix.tex
\section{Additional Details}

Additional details can be placed in the appendix.

\subsection{Proof of Lemma~\ref{thm:existence}}
\label{app:proof}

\begin{proof}[Proof of Lemma~\ref{thm:existence}]
Write $D \equiv D_\tau$. We give an explicit construction; throughout, a
sample is identified with its \emph{flip set} $T \subseteq \{1,\dots,D\}$
of zeroed dimensions, so its label vector is $y_d = \mathbf{1}[d \notin T]$
and its score is $s = D - |T|$.

\noindent\emph{Step~1 (target multiset of flip sizes).} Score
uniformity~\eqref{eq:score_balance} requires exactly $M$ samples at each
score $s \in \{0,\dots,D\}$, i.e.\ $M$ flip sets of size $D-s$ for each
$s$. Enumerate these $M(D+1)$ required sizes in any order as
$\ell_1,\dots,\ell_{M(D+1)}$, where the value $D-s$ occurs $M$ times for
each $s$. Their total is
\[
T_{\mathrm{tot}} \;=\; \sum_{r}\ell_r \;=\; M\sum_{s=0}^{D}(D-s) \;=\; M\,\frac{D(D+1)}{2}.
\]

\noindent\emph{Step~2 (cyclic-arc placement).} Identify the dimensions
with $\mathbb{Z}_D$. Assign to the $r$-th sample the length-$\ell_r$
cyclic arc
\[
T_r \;=\; \bigl\{\,(a_r + k)\bmod D : k = 0,\dots,\ell_r-1\,\bigr\}+1,
\qquad a_r \;=\; \Bigl(\textstyle\sum_{q<r}\ell_q\Bigr)\bmod D,
\]
so that each arc starts exactly where the previous one ended and the
arcs tile the cycle end-to-end. By (A1) pick any seed
$(p_0^+,x_0^+)\in\mathcal{S}_{\mathrm{full}}^{(\tau)}$; by (A2)+(A3) the
composite operator $\mathcal{F}_{T_r}=\bigcirc_{d\in T_r}\mathcal{F}_d$
realizes the flip set exactly, yielding a sample with
$y_d=\mathbf{1}[d\notin T_r]$ and score $D-\ell_r$. Collect all $M(D+1)$
samples into $\mathcal{B}$.

\noindent\emph{Step~3 (both constraints hold).} By Step~1 each score
$s$ is realized exactly $M$ times, so~\eqref{eq:score_balance} holds.
For~\eqref{eq:dim_balance}, note that laying consecutive arcs of total
length $T_{\mathrm{tot}}$ around a cycle of size $D$ covers each of the
$D$ positions exactly $T_{\mathrm{tot}}/D$ times whenever
$D \mid T_{\mathrm{tot}}$. Here $T_{\mathrm{tot}}/D = M(D+1)/2$, which is
an integer precisely because $M(D+1)$ is even (the lemma's
hypothesis). Hence every dimension is flipped (labeled $0$) exactly
$M(D+1)/2$ times, so it is labeled $1$ exactly
$M(D+1)-M(D+1)/2 = M(D+1)/2 = N/2$ times, which
is~\eqref{eq:dim_balance}.

\noindent\emph{Step~4 (consistency).} The two constraints imply the
same total positive count, $\sum_d N/2 = \tfrac{MD(D+1)}{2} = \sum_{s=0}^{D} M s$,
confirming that they are mutually compatible and jointly satisfied by
the explicit $\mathcal{B}$ above. \qedhere
\end{proof}

\subsection{Full Evaluation Dimension Tables}
\label{app:dimension_tables}

The complete lists of evaluation dimensions, with category assignments and per-dimension descriptions, are given below for the three generation scenarios used in this work. These tables enumerate the atomic binary requirements that constitute the decoupled taxonomy described in Section~\ref{sec:dimensions}.

\input{table_t2i_dimensions}
\input{table_t2v_dimensions}
\input{table_tts_dimensions}

\subsection{Modality-Side Atomic Operators}
\label{app:atomic_operators}

This subsection lists, by task, the modality-side atomic operators $\mathcal{O}_{d^*}^{(\tau)}$ used in Step~3 of \S\ref{subsec:method_decoupling} to flip a single modality-related dimension while leaving every other dimension intact. Each operator was validated on a held-out audit set to confirm dimensional isolation.

\paragraph{T2V atomic operators (6).}
\begin{enumerate}
  \item Subject / background warping, jitter and drift via SAM3-based instance segmentation with adaptive in-painting;
  \item Texture-focus dynamic blur via bilateral filtering driven by a dynamic high-frequency-region detector;
  \item Frame-order shuffling and inter-frame flicker via randomized frame whitening and re-ordering;
  \item Implausible re-timing via $2\!\times$ speed masking on the audio-video stream;
  \item Audio-track Gaussian noise injection;
  \item Audio--visual desynchronization via audio-track separation and temporal offset.
\end{enumerate}

\paragraph{T2I atomic operators (3).}
\begin{enumerate}
  \item Greasy-texture distortion via frequency-domain filtering;
  \item Image blurring via Gaussian blur;
  \item Anatomical (limb / joint) deformation via keypoint detection followed by localized smearing.
\end{enumerate}

\paragraph{TTS atomic operators (4).}
\begin{enumerate}
  \item Muffled voice via bilateral filtering on the spectrogram;
  \item Background reverberation via audio replay plus additive noise;
  \item Volume instability via random smoothed gain modulation;
  \item High-frequency squeal / low-frequency mud via targeted frequency-domain manipulation.
\end{enumerate}

\clearpage
\subsection{Evaluation Prompt Templates}
\label{app:eval_prompts}

This subsection reproduces the verbatim English prompt templates used to elicit every per-dimension Yes/No score reported in Section~\ref{sec:experiments}. Each prompt is sent to every judge of the corresponding task without modification; the only field substituted at evaluation time is the user-side prompt (and, for TTS, the speaker instruction and the spoken text).

\subsubsection{Text-to-Image Judging Prompt (17 dimensions)}
\label{app:eval_t2i}

\begin{lstlisting}
You are a Text-to-Image quality evaluation expert. Please carefully examine the provided image and objectively evaluate the quality of the generated image based on the corresponding image generation prompt.

## Evaluation Task
Please evaluate based on the following inputs:
- **Image Generation Prompt**: {{image_generation_prompt}}
- **Image to Evaluate**: The provided image

After carefully examining the image, assess each of the following 17 dimensions to determine whether the image content meets the requirements of the generation prompt and quality standards:

- Dimension 1 Composition & Focus: Does the composition of the image match the prompt?
- Dimension 2 Color Performance: Does the color palette of the image match the prompt?
- Dimension 3 Lighting & Atmosphere: Does the lighting effectively create the specific atmosphere described in the prompt?
- Dimension 4 Emotional Expression: Does the overall emotion conveyed by the image precisely match the emotional requirement specified in the prompt?
- Dimension 5 Basic Attributes: Does the core identity of the subject--including count, category, gender, and other basic attributes--match the prompt?
- Dimension 6 Appearance: Do the visual appearance details of the subject--such as color, material, shape, clothing, and hairstyle--match the prompt?
- Dimension 7 Behavior: Does the subject's state and behavior--including facial expression, posture, and action--match the prompt?
- Dimension 8 Negative Prompt: Does the image successfully avoid elements explicitly excluded in the prompt (e.g., "no red", "no cars")?
- Dimension 9 Spatial Relationship: Do the spatial positions, layout, and occlusion relationships among objects or subjects conform to the prompt?
- Dimension 10 Scene Type: Does the scene type in the image match the one specified in the prompt?
- Dimension 11 Style & Aesthetic Adherence: Does the image accurately reflect the specific artistic style, design movement, or aesthetic system specified in the prompt?
- Dimension 12 Text Accuracy: Is any text in the image clear, legible, and free of garbled characters or typos?
- Dimension 13 Text Font/Typesetting: Do the font choice and typesetting (e.g., centering, line spacing) of the text match the prompt requirements?
- Dimension 14 Shot Requirement: Does the image accurately reflect the shot size, camera angle, or specific photographic effect (e.g., bokeh, fisheye) specified in the prompt?
- Dimension 15 Material & Detail: Is the surface material--such as human skin, fabric, or metal--rendered with realistic texture and rich detail?
- Dimension 16 Edge & Clarity: Are object contours and edges sharp, clean, and free from unnatural blurring or blending?
- Dimension 17 Object Anatomy: Are the structures of subjects--whether objects or people--natural and plausible? For humans, are facial proportions, body anatomy, and hands accurate and anatomically correct?

## Output Format
Provide a Yes or No answer for each dimension in order, formatted as a JSON array. Ensure the results are arranged in the order of the dimensions:
["Yes", "No", "Yes", ...] (17 elements in total)

## Notes
1. You must strictly return the result in JSON array format, without any additional explanatory text
2. The array must contain exactly 17 elements, each being either "Yes" or "No"
3. Please first examine the image thoroughly, then evaluate each item against the prompt and quality requirements
4. The evaluation must be objective, based on the degree of match between the actual image content and the prompt, as well as the actual image quality
\end{lstlisting}

\subsubsection{Text-to-Video Judging Prompt (22 dimensions)}
\label{app:eval_t2v}

\begin{lstlisting}
You are a video quality assessment expert. Please carefully watch the provided video and objectively evaluate the quality of the generated video based on the corresponding video generation prompt.

## Evaluation Task
Please evaluate based on the following two inputs:
- **Video Generation Prompt**: {{video_generation_prompt}}
- **Video to Evaluate**: The provided video

After carefully watching the video, please evaluate the video content against the generation prompt requirements and quality requirements based on the following 22 dimensions:

- Dimension 1 Subject Type: Does the type of main subject in the video match the one specified in the prompt?
- Dimension 2 Subject Count: Does the number of main subjects in the video match the count specified in the prompt?
- Dimension 3 Completeness (Extra/Missing): Does the video contain all required elements and no extra or missing subjects as described in the prompt?
- Dimension 4 Appearance (Color/Shape/Clothing): Do the visual attributes of the subject--such as color, shape, size, and clothing--match those specified in the prompt?
- Dimension 5 Action/Behavior/State: Does the subject perform the correct action, behavior, or state described in the prompt?
- Dimension 6 Background Scene Type: Does the type of background scene in the video match the one specified in the prompt?
- Dimension 7 Subject-Background Spatial Relation: Is the spatial relationship between the subject and the background consistent with the prompt?
- Dimension 8 Inter-Subject Relative Position: Are the relative positions among multiple subjects (or key visual elements) consistent with the prompt?
- Dimension 9 Subject-Camera Position/Framing: Is the subject's position relative to the camera (e.g., framing, distance, orientation) consistent with the prompt?
- Dimension 10 Shot Parameters (Shot/Depth/Angle/Motion): Do the shot parameters--including shot size (e.g., close-up, wide), depth of field, camera angle, and motion--match those specified in the prompt?
- Dimension 11 Lighting (Direction/Source/Quality): Does the lighting direction, source, and quality (e.g., soft/hard, natural/artificial) match the prompt?
- Dimension 12 Color Tone (Hue/Sat/Bright): Are the overall hue, saturation, and brightness of the video appropriate and consistent with the prompt?
- Dimension 13 Video Style/Atmosphere: Does the visual style and overall atmosphere of the video match the one described in the prompt?
- Dimension 14 Audio Content Type: Does the type of audio content in the video (e.g., speech, music, ambient sound) match the prompt?
- Dimension 15 Audio Style: Does the stylistic character of the audio (e.g., cheerful, somber, tense) align with the prompt?
- Dimension 16 Audio Parameters: Do the audio parameters--such as volume, speaking rate, or instrument choice--match those specified in the prompt?
- Dimension 17 Video Reality: Whether video has no unreasonable errors in subject or background, such as distortion, deformation, drift, or jitter?
- Dimension 18 Video Focal: Whether video is accurately focused with clear texture and no unreasonable blur?
- Dimension 19 Video Continuity: Whether video has no frame-to-frame flickering?
- Dimension 20 Video Rhythm: Whether action rhythm is reasonable (no sudden speed changes)?
- Dimension 21 Audio Quality: Whether audio in video is clear with background noise controlled reasonably?
- Dimension 22 Audio-Video Sync: Whether audio and video are synchronized?

## Output Format
Please provide Yes or No answers for each dimension in order (22 answers total, only use Yes or No, answer Yes if not violated or not applicable), in JSON array format, ensuring results are arranged in dimension order:
["Yes", "No", "Yes", ...] (22 elements total)

## Notes
1. Must return strictly in JSON array format, do not include any other explanatory text
2. Array must contain exactly 22 elements, each element can only be "Yes" or "No"
3. Please watch the video completely first, then evaluate each item independently against the prompt and quality requirements
4. Evaluation must be objective, based on the match between actual video content and prompt, and actual video quality
\end{lstlisting}

\subsubsection{Text-to-Speech Judging Prompt (14 dimensions)}
\label{app:eval_tts}

\begin{lstlisting}
You are a TTS generated voice quality evaluation expert. Please carefully listen to the provided audio, and objectively evaluate the quality of the generated voice audio based on the corresponding audio generation prompts (speaker setting instruction <prompt_instruction> and speaking content <prompt_text>).

## Evaluation Task
Please evaluate based on the following inputs:
- **Audio Generation Prompt: Speaker Setting Instruction**: <prompt_instruction>
- **Audio Generation Prompt: Speaking Content**: <prompt_text>
- **Audio to Evaluate**: The provided audio

Please carefully listen to the audio, and evaluate whether the audio content and quality meet the requirements of the generation prompts and quality standards across the following 14 dimensions:

- Dimension 1 Text Consistency: Does the spoken content in the audio match the text specified in the Audio Generation Prompt: Speaking Content?
- Dimension 2 Punctuation Consistency: Does the phrasing and pausing in the audio reflect the punctuation and sentence structure in the Audio Generation Prompt: Speaking Content?
- Dimension 3 Subject Age: Does the speaker's perceived age in the audio match the age specified in the Audio Generation Prompt: Speaker Setting Instruction?
- Dimension 4 Subject Gender: Does the speaker's perceived gender in the audio match the gender specified in the Audio Generation Prompt: Speaker Setting Instruction?
- Dimension 5 Subject Personality: Does the speaker's vocal expression convey the personality described in the Audio Generation Prompt: Speaker Setting Instruction?
- Dimension 6 Subject Timbre: Does the speaker's voice timbre align with the timbre described or implied in the Audio Generation Prompt: Speaker Setting Instruction?
- Dimension 7 Subject Speed: Does the speaking rate in the audio match the speed specified or implied in the Audio Generation Prompt: Speaker Setting Instruction?
- Dimension 8 Subject Pitch: Does the speaker's vocal pitch match the pitch described or implied in the Audio Generation Prompt: Speaker Setting Instruction?
- Dimension 9 Subject Tone: Does the intonation and vocal attitude align with the tone specified in the Audio Generation Prompt: Speaker Setting Instruction?
- Dimension 10 Subject Emotion: Does the emotional state expressed in the voice match the emotion specified in the Audio Generation Prompt: Speaker Setting Instruction?
- Dimension 11 Audio Clarity: Is the human voice in the audio clear?
- Dimension 12 Audio Background: Is there reverberation or background noise in the audio?
- Dimension 13 Audio Volume Stability: Is the audio volume stable?
- Dimension 14 Audio Frequency: Are there any anomalies in the audio frequency?

## Output Format
Please provide a Yes or No answer for each dimension in order, formatted as a JSON array. Ensure the results are arranged in the order of the dimensions:
["Yes", "No", "Yes", ...] (14 elements in total)

## Notes
1. Must strictly return in JSON array format, without any other explanatory text
2. The array must contain exactly 14 elements, each of which can only be "Yes" or "No". If the audio content does not contain the situation described by the dimension at all, answer "Yes"
3. Please listen to the audio completely first, then evaluate each item against the prompts and the audio's own quality
4. The evaluation must be objective, based on the degree of matching between the actual audio content and the prompts, and the actual quality of the audio

# User Input to Process
<prompt_instruction>{{prompt_instruction}}</prompt_instruction>
<prompt_text>{{prompt_text}}</prompt_text>
\end{lstlisting}

\subsection{Polarity-Resolved Judging Diagnostics}
\label{app:polarity}

The two diagnostics below refine the main-text analysis by splitting
every per-dimension quantity by the ground-truth polarity, separating a
judge's ability to \emph{confirm} satisfied requirements (the Yes side)
from its ability to \emph{detect} violated ones (the No side). Both use
the English runs; the Chinese counterparts are qualitatively identical.

\paragraph{Per-dimension Yes/No accuracy.}
Figure~\ref{fig:dim_yesno} reports, for every judge and dimension, the
Yes accuracy (recall on ground-truth-Yes dimensions) and the No accuracy
(specificity on ground-truth-No dimensions), a per-dimension
decomposition of the perfect-match view of Figure~\ref{fig:perfect} and
of the radii of Figure~\ref{fig:dim}. It shows that the near-chance
radii of the radar are not unbiased guessing: on essentially every
near-$50\%$ dimension the two rows are far apart (for example T2V Video
Focal at Yes $0.99$/No $0.02$ and TTS Volume Stability at $0.99$/$0.03$),
the signature of a near-constant Yes prediction rather than a random
$50/50$ split.

\begin{table}[H]
\centering
\small
\begin{tabular}{l l c c}
\toprule
T2V audio dimension & Type & Yes acc. & No acc. \\
\midrule
Audio Content Type & prompt-alignment & $0.51$ & $0.61$ \\
Audio Style        & prompt-alignment & $0.52$ & $0.63$ \\
Audio Params       & prompt-alignment & $0.51$ & $0.62$ \\
Audio Quality      & modality-quality & $0.66$ & $0.26$ \\
Audio-Video Sync   & modality-quality & $0.66$ & $0.27$ \\
\bottomrule
\end{tabular}
\caption{Qwen-3.6-27B, a judge without an audio channel, on the five
T2V audio dimensions, split by ground-truth polarity. On the
prompt-alignment dimensions it stays near chance on both rows (it cannot
confirm whether the audio matches the prompt), whereas on the intrinsic
quality and sync dimensions it defaults to ``Yes'' (Yes accuracy
$0.66$, No accuracy $\sim\!0.26$), rarely flagging a defect it cannot
hear. The two behaviors give the same near-chance radius in
Figure~\ref{fig:dim} for opposite reasons.}
\label{tab:qwen36_audio}
\end{table}

\paragraph{Polarity-conditioned coupling.}
Figure~\ref{fig:corr_yesno} recomputes the off-diagonal correlation
matrix of Figure~\ref{fig:corr} separately on the dimension pairs whose
ground truth is Yes on both axes (the confirmation-side halo) and No on
both axes (the defect-side halo), attributing each coupling in the mixed
matrix to a polarity. On T2I the coupling is almost entirely a
confirmation-side halo ($31$ coupled dimension pairs are Yes-driven
against $4$ on the No side), and the only defect-side coupling is
confined to the fine-grained modality block (Material, Edge \& Clarity,
Object Anatomy); on T2V both sides couple heavily ($61$ Yes-side
against $77$ No-side pairs) and on TTS they are balanced ($22$ each).
The split also surfaces couplings that the mixed matrix averages away,
and judges that look independent overall yet couple on a single polarity
(for example Gemini-3-Flash on TTS Audio-Clarity-with-Volume-Stability,
$0.25$ in the mixed matrix but $0.71$ on the Yes side).

%% file: table_t2i_dimensions.tex
\begin{table}[H]
\centering
\resizebox{\textwidth}{!}{
\begin{tabular}{llp{10cm}}
\toprule
\textbf{Category} & \textbf{Dimension} & \textbf{Description} \\
\midrule
\multirow{14}{*}{Prompt-related}
& Composition \& Focus & Does the composition of the image match the prompt? \\
& Color Performance & Does the color palette of the image match the prompt? \\
& Lighting \& Atmosphere & Does the lighting effectively create the specific atmosphere described in the prompt? \\
& Emotional Expression & Does the overall emotion conveyed by the image precisely match the emotional requirement specified in the prompt? \\
& Basic Attributes & Does the core identity of the subject (including count, category, gender, and other basic attributes) match the prompt? \\
& Appearance & Do the visual appearance details of the subject (such as color, material, shape, clothing, and hairstyle) match the prompt? \\
& Behavior & Does the subject's state and behavior (including facial expression, posture, and action) match the prompt? \\
& Negative Prompt & Does the image successfully avoid elements explicitly excluded in the prompt (e.g., ``no red'', ``no cars'')? \\
& Spatial Relationship & Do the spatial positions, layout, and occlusion relationships among objects or subjects conform to the prompt? \\
& Scene Type & Does the scene type in the image match the one specified in the prompt? \\
& Style \& Aesthetic Adherence & Does the image accurately reflect the specific artistic style, design movement, or aesthetic system specified in the prompt? \\
& Text Accuracy & Is any text in the image clear, legible, and free of garbled characters or typos? \\
& Text Font/Typesetting & Do the font choice and typesetting (e.g., centering, line spacing) of the text match the prompt requirements? \\
& Shot Requirement & Does the image accurately reflect the shot size, camera angle, or specific photographic effect (e.g., bokeh, fisheye) specified in the prompt? \\
\midrule
\multirow{3}{*}{Modality-related}
& Material \& Detail & Is the surface material (such as human skin, fabric, or metal) rendered with realistic texture and rich detail? \\
& Edge \& Clarity & Are object contours and edges sharp, clean, and free from unnatural blurring or blending? \\
& Object Anatomy & Are the structures of subjects (whether objects or people) natural and plausible? For humans, are facial proportions, body anatomy, and hands accurate and anatomically correct? \\
\bottomrule
\end{tabular}
}
\caption{Evaluation dimensions for Text-to-Image (T2I) generation. The 17 dimensions are divided into prompt-related dimensions that assess semantic alignment with the input prompt, and modality-related dimensions that evaluate intrinsic image quality.}
\label{tab:t2i_dimensions}
\end{table}

%% file: table_t2v_dimensions.tex
\begin{table}[H]
\centering
\resizebox{\textwidth}{!}{
\begin{tabular}{llp{10cm}}
\toprule
\textbf{Category} & \textbf{Dimension} & \textbf{Description} \\
\midrule
\multirow{16}{*}{Prompt-related}
& Subject Type & Does the type of main subject in the video match the one specified in the prompt? \\
& Subject Count & Does the number of main subjects in the video match the count specified in the prompt? \\
& Completeness: Extra/Missing & Does the video contain all required elements and no extra or missing subjects as described in the prompt? \\
& Appearance: Color/Shape/Clothing & Do the visual attributes of the subject (such as color, shape, size, and clothing) match those specified in the prompt? \\
& Action/Behavior/State & Does the subject perform the correct action, behavior, or state described in the prompt? \\
& Background Scene Type & Does the type of background scene in the video match the one specified in the prompt? \\
& Subject-Background Spatial Relation & Is the spatial relationship between the subject and the background consistent with the prompt? \\
& Inter-Subject Relative Position & Are the relative positions among multiple subjects (or key visual elements) consistent with the prompt? \\
& Subject-Camera Position/Framing & Is the subject's position relative to the camera (e.g., framing, distance, orientation) consistent with the prompt? \\
& Shot/Depth/Angle/Motion & Do the shot parameters (including shot size, depth of field, camera angle, and motion) match those specified in the prompt? \\
& Lighting: Direction/Source/Quality & Does the lighting direction, source, and quality (e.g., soft/hard, natural/artificial) match the prompt? \\
& Color Tone: Hue/Sat/Bright & Are the overall hue, saturation, and brightness of the video appropriate and consistent with the prompt? \\
& Video Style/Atmosphere & Does the visual style and overall atmosphere of the video match the one described in the prompt? \\
& Audio Content Type & Does the type of audio content in the video (e.g., speech, music, ambient sound) match the prompt? \\
& Audio Style & Does the stylistic character of the audio (e.g., cheerful, somber, tense) align with the prompt? \\
& Audio Params & Do the audio parameters (such as volume, speaking rate, or instrument choice) match those specified in the prompt? \\
\midrule
\multirow{6}{*}{Modality-related}
& Video Reality & Whether the video has no unreasonable errors in subject or background, such as distortion, deformation, drift, or jitter? \\
& Video Focal & Whether the video is accurately focused with clear texture and no unreasonable blur? \\
& Video Continuity & Whether the video has no frame-to-frame flickering? \\
& Video Rhythm & Whether the action rhythm is reasonable (no sudden speed changes)? \\
& Audio Quality & Whether the audio in the video is clear with background noise controlled reasonably? \\
& Audio-Video Sync & Whether the audio and video are synchronized? \\
\bottomrule
\end{tabular}
}
\caption{Evaluation dimensions for Text-to-Video (T2V) generation. The 22 dimensions are divided into prompt-related dimensions that assess semantic alignment with the input prompt, and modality-related dimensions that evaluate intrinsic video and audio quality.}
\label{tab:t2v_dimensions}
\end{table}

%% file: table_tts_dimensions.tex
\begin{table}[H]
\centering
\resizebox{\textwidth}{!}{
\begin{tabular}{llp{10cm}}
\toprule
\textbf{Category} & \textbf{Dimension} & \textbf{Description} \\
\midrule
\multirow{10}{*}{Prompt-related}
& Text Consistency & Does the spoken content in the audio match the text specified in the prompt? \\
& Punctuation Consistency & Does the phrasing and pausing in the audio reflect the punctuation and sentence structure in the prompt? \\
& Subject Age & Does the speaker's perceived age in the audio match the age specified in the prompt instruction? \\
& Subject Gender & Does the speaker's perceived gender in the audio match the gender specified in the prompt instruction? \\
& Subject Personality & Does the speaker's vocal expression convey the personality described in the prompt instruction? \\
& Subject Timbre & Does the speaker's voice timbre align with the timbre described or implied in the prompt instruction? \\
& Subject Speed & Does the speaking rate in the audio match the speed specified or implied in the prompt instruction? \\
& Subject Pitch & Does the speaker's vocal pitch match the pitch described or implied in the prompt instruction? \\
& Subject Tone & Does the intonation and vocal attitude align with the tone specified in the prompt instruction? \\
& Subject Emotion & Does the emotional state expressed in the voice match the emotion specified in the prompt instruction? \\
\midrule
\multirow{4}{*}{Modality-related}
& Audio Clarity & Is the human voice in the audio clear? \\
& Audio Background & Is there reverberation or background noise in the audio? \\
& Audio Volume Stability & Is the audio volume stable? \\
& Audio Frequency & Are there any anomalies in the audio frequency? \\
\bottomrule
\end{tabular}
}
\caption{Evaluation dimensions for Text-to-Speech (TTS) generation. The 14 dimensions are divided into prompt-related dimensions that assess alignment with the speaker instruction and text content, and modality-related dimensions that evaluate intrinsic audio quality.}
\label{tab:tts_dimensions}
\end{table}

%% file: main.bib
@article{li2026omnibench,
  title={OmniBench: Towards The Future of Universal Omni-Language Models},
  author={Li, Yizhi and Zhang, Ge and Ma, Yinghao and Yuan, Ruibin and Zhu, Kang and Guo, Hangyu and Liang, Yiming and Liu, Jiaheng and Wang, Zekun and Yang, Jian and Wu, Siwei and others},
  journal={Advances in Neural Information Processing Systems},
  volume={38},
  year={2025}
}

@article{zhang2025omnieval,
  title={OmniEval: A Benchmark for Evaluating Omni-modal Models with Visual, Auditory, and Textual Inputs},
  author={Zhang, Yiman and Luo, Ziheng and Yan, Qiangyu and He, Wei and Jiang, Borui and Chen, Xinghao and Han, Kai},
  journal={arXiv preprint arXiv:2506.20960},
  year={2025}
}

@article{zheng2023judging,
  title={Judging LLM-as-a-Judge with MT-Bench and Chatbot Arena},
  author={Zheng, Lianmin and Chiang, Wei-Lin and Sheng, Ying and Zhuang, Siyuan and Wu, Zhanghao and Zhuang, Yonghao and Lin, Zi and Li, Zhuohan and Li, Dacheng and Xing, Eric and others},
  journal={Advances in Neural Information Processing Systems},
  volume={36},
  pages={46595--46623},
  year={2023}
}

@inproceedings{lambert2025rewardbench,
  title={RewardBench: Evaluating Reward Models for Language Modeling},
  author={Lambert, Nathan and Pyatkin, Valentina and Morrison, Jacob and Miranda, LJ and Lin, Bill Yuchen and Chandu, Khyathi and Dziri, Nouha and Kumar, Sachin and Zick, Tom and Choi, Yejin and others},
  booktitle={Findings of the Association for Computational Linguistics: NAACL 2025},
  pages={1755--1797},
  year={2025}
}

@article{yasunaga2025multimodal,
  title={Multimodal RewardBench: Holistic Evaluation of Reward Models for Vision Language Models},
  author={Yasunaga, Michihiro and Zettlemoyer, Luke and Ghazvininejad, Marjan},
  journal={arXiv preprint arXiv:2502.14191},
  year={2025}
}

@inproceedings{xiong2025llava,
  title={LLaVA-Critic: Learning to Evaluate Multimodal Models},
  author={Xiong, Tianyi and Wang, Xiyao and Guo, Dong and Ye, Qinghao and Fan, Haoqi and Gu, Quanquan and Huang, Heng and Li, Chunyuan},
  booktitle={Proceedings of the Computer Vision and Pattern Recognition Conference},
  pages={13618--13628},
  year={2025}
}

@article{jin2025omni,
  title={Omni-Reward: Towards Generalist Omni-Modal Reward Modeling with Free-Form Preferences},
  author={Jin, Zhuoran and Yuan, Hongbang and Zhu, Kejian and Li, Jiachun and Cao, Pengfei and Chen, Yubo and Liu, Kang and Zhao, Jun},
  journal={arXiv preprint arXiv:2510.23451},
  year={2025}
}

@article{wijaya2024multimodal,
  title={Multimodal Preference Data Synthetic Alignment with Reward Model},
  author={Wijaya, Robert and Nguyen, Ngoc-Bao and Cheung, Ngai-Man},
  journal={arXiv preprint arXiv:2412.17417},
  year={2024}
}

@article{chen2026advancing,
  title={Advancing Multimodal Judge Models through a Capability-Oriented Benchmark and MCTS-Driven Data Generation},
  author={Chen, Zeyu and Yao, Huanjin and Zhao, Ziwang and Yang, Min},
  journal={arXiv preprint arXiv:2603.00546},
  year={2026}
}

@misc{claude35sonnet,
  title = {Claude 3.5 Sonnet},
  author = {Anthropic},
  url = {https://www.anthropic.com/news/claude-3-5-sonnet},
  year={2024},
  note={Accessed: 2026-05-21}
}

@inproceedings{wang2024pandalm,
  title={PandaLM: An Automatic Evaluation Benchmark for LLM Instruction Tuning Optimization},
  author={Wang, Yidong and Yu, Zhuohao and Zeng, Zhengran and Yang, Linyi and Wang, Cunxiang and Chen, Hao and Jiang, Chaoya and Xie, Rui and Wang, Jindong and Xie, Xing and others},
  booktitle={The Twelfth International Conference on Learning Representations},
  year={2024}
}

@inproceedings{zhu2023judgelm,
  title={JudgeLM: Fine-tuned Large Language Models are Scalable Judges},
  author={Zhu, Lianghui and Wang, Xinggang and Wang, Xinlong},
  booktitle={The Thirteenth International Conference on Learning Representations},
  year={2025}
}

@inproceedings{li2023generative,
  title={Generative Judge for Evaluating Alignment},
  author={Li, Junlong and Sun, Shichao and Yuan, Weizhe and Fan, Run-Ze and Zhao, Hai and Liu, Pengfei},
  booktitle={The Twelfth International Conference on Learning Representations},
  year={2024}
}

@inproceedings{kim2023prometheus,
  title={Prometheus: Inducing Fine-Grained Evaluation Capability in Language Models},
  author={Kim, Seungone and Shin, Jamin and Cho, Yejin and Jang, Joel and Longpre, Shayne and Lee, Hwaran and Yun, Sangdoo and Shin, Seongjin and Kim, Sungdong and Thorne, James and others},
  booktitle={The Twelfth International Conference on Learning Representations},
  year={2024}
}

@article{cao2024compassjudger,
  title={CompassJudger-1: All-in-One Judge Model Helps Model Evaluation and Evolution},
  author={Cao, Maosong and Lam, Alexander and Duan, Haodong and Liu, Hongwei and Zhang, Songyang and Chen, Kai},
  journal={arXiv preprint arXiv:2410.16256},
  year={2024}
}

@article{xu2023imagereward,
  title={ImageReward: Learning and Evaluating Human Preferences for Text-to-Image Generation},
  author={Xu, Jiazheng and Liu, Xiao and Wu, Yuchen and Tong, Yuxuan and Li, Qinkai and Ding, Ming and Tang, Jie and Dong, Yuxiao},
  journal={Advances in Neural Information Processing Systems},
  volume={36},
  pages={15903--15935},
  year={2023}
}

@article{kirstain2023pick,
  title={Pick-a-Pic: An Open Dataset of User Preferences for Text-to-Image Generation},
  author={Kirstain, Yuval and Polyak, Adam and Singer, Uriel and Matiana, Shahbuland and Penna, Joe and Levy, Omer},
  journal={Advances in Neural Information Processing Systems},
  volume={36},
  pages={36652--36663},
  year={2023}
}

@article{wu2023human,
  title={Human Preference Score v2: A Solid Benchmark for Evaluating Human Preferences of Text-to-Image Synthesis},
  author={Wu, Xiaoshi and Hao, Yiming and Sun, Keqiang and Chen, Yixiong and Zhu, Feng and Zhao, Rui and Li, Hongsheng},
  journal={arXiv preprint arXiv:2306.09341},
  year={2023}
}

@inproceedings{zhang2024learning,
  title={Learning Multi-Dimensional Human Preference for Text-to-Image Generation},
  author={Zhang, Sixian and Wang, Bohan and Wu, Junqiang and Li, Yan and Gao, Tingting and Zhang, Di and Wang, Zhongyuan},
  booktitle={Proceedings of the IEEE/CVF Conference on Computer Vision and Pattern Recognition},
  pages={8018--8027},
  year={2024}
}

@inproceedings{lo2019mosnet,
  title={MOSNet: Deep Learning-Based Objective Assessment for Voice Conversion},
  author={Lo, Chen-Chou and Fu, Szu-Wei and Huang, Wen-Chin and Wang, Xin and Yamagishi, Junichi and Tsao, Yu and Wang, Hsin-Min},
  booktitle={Interspeech 2019},
  pages={1541--1545},
  year={2019},
  organization={ISCA}
}

@inproceedings{saeki2022utmos,
  title={UTMOS: UTokyo-SaruLab System for VoiceMOS Challenge 2022},
  author={Saeki, Takaaki and Xin, Detai and Nakata, Wataru and Koriyama, Tomoki and Takamichi, Shinnosuke and Saruwatari, Hiroshi},
  booktitle={Interspeech 2022},
  pages={4521--4525},
  year={2022},
  organization={ISCA}
}

@inproceedings{maiti2023speechlmscore,
  title={SpeechLMScore: Evaluating Speech Generation Using Speech Language Model},
  author={Maiti, Soumi and Peng, Yifan and Saeki, Takaaki and Watanabe, Shinji},
  booktitle={IEEE International Conference on Acoustics, Speech and Signal Processing (ICASSP)},
  pages={1--5},
  year={2023},
  organization={IEEE}
}

@inproceedings{pu2025judge,
  title={Judge Anything: MLLM as a Judge Across Any Modality},
  author={Pu, Shu and Wang, Yaochen and Chen, Dongping and Chen, Yuhang and Wang, Guohao and Qin, Qi and Zhang, Zhongyi and Zhang, Zhiyuan and Zhou, Zetong and Gong, Shuang and others},
  booktitle={Proceedings of the 31st ACM SIGKDD Conference on Knowledge Discovery and Data Mining V. 2},
  pages={5742--5753},
  year={2025}
}

@inproceedings{chen2024mllm,
  title={MLLM-as-a-Judge: Assessing Multimodal LLM-as-a-Judge with Vision-Language Benchmark},
  author={Chen, Dongping and Chen, Ruoxi and Zhang, Shilin and Wang, Yaochen and Liu, Yinuo and Zhou, Huichi and Zhang, Qihui and Wan, Yao and Zhou, Pan and Sun, Lichao},
  booktitle={Forty-first International Conference on Machine Learning},
  year={2024}
}

@inproceedings{huang2024vbench,
  title={VBench: Comprehensive Benchmark Suite for Video Generative Models},
  author={Huang, Ziqi and He, Yinan and Yu, Jiashuo and Zhang, Fan and Si, Chenyang and Jiang, Yuming and Zhang, Yuanhan and Wu, Tianxing and Jin, Qingyang and Chanpaisit, Nattapol and others},
  booktitle={Proceedings of the IEEE/CVF Conference on Computer Vision and Pattern Recognition},
  pages={21807--21818},
  year={2024}
}

@inproceedings{liu2024evalcrafter,
  title={EvalCrafter: Benchmarking and Evaluating Large Video Generation Models},
  author={Liu, Yaofang and Cun, Xiaodong and Liu, Xuebo and Wang, Xintao and Zhang, Yong and Chen, Haoxin and Liu, Yang and Zeng, Tieyong and Chan, Raymond and Shan, Ying},
  booktitle={Proceedings of the IEEE/CVF Conference on Computer Vision and Pattern Recognition},
  pages={22139--22149},
  year={2024}
}

@article{wu2024t2vscore,
  title={Towards A Better Metric for Text-to-Video Generation},
  author={Wu, Jay Zhangjie and Fang, Guian and Wu, Haoning and Wang, Xintao and Ge, Yixiao and Cun, Xiaodong and Zhang, David Junhao and Liu, Jia-Wei and Gu, Yuchao and Liu, Rui and others},
  journal={arXiv preprint arXiv:2401.07781},
  year={2024}
}

@inproceedings{huang2023t2icompbench,
  title={T2I-CompBench: A Comprehensive Benchmark for Open-world Compositional Text-to-image Generation},
  author={Huang, Kaiyi and Sun, Kaiyue and Xie, Enze and Li, Zhenguo and Liu, Xihui},
  booktitle={Advances in Neural Information Processing Systems},
  year={2023}
}

@inproceedings{lin2024genaibench,
  title={GenAI-Bench: Evaluating and Improving Compositional Text-to-Visual Generation},
  author={Li, Baiqi and Lin, Zhiqiu and Pathak, Deepak and Li, Jiayao and Fei, Yixin and Wu, Kewen and Ling, Tiffany and Xia, Xide and Zhang, Pengchuan and Neubig, Graham and Ramanan, Deva},
  booktitle={Synthetic Data for Computer Vision Workshop at CVPR},
  year={2024}
}

@misc{openai2025gpt5,
  title={GPT-5},
  author={OpenAI},
  year={2025},
  url={https://openai.com/gpt-5/},
  note={Accessed: 2026-05-21}
}

@misc{anthropic2025claude4,
  title={Introducing Claude Opus 4.6},
  author={Anthropic},
  year={2026},
  url={https://www.anthropic.com/news/claude-opus-4-6},
  note={Accessed: 2026-05-21}
}

@misc{google2025gemini3,
  title={A New Era of Intelligence with Gemini 3},
  author={Google},
  year={2025},
  url={https://blog.google/products-and-platforms/products/gemini/gemini-3/},
  note={Accessed: 2026-05-21}
}

@article{hurst2024gpt,
  title={GPT-4o System Card},
  author={Hurst, Aaron and Lerer, Adam and Goucher, Adam P and Perelman, Adam and Ramesh, Aditya and Clark, Aidan and Ostrow, AJ and Welihinda, Akila and Hayes, Alan and Radford, Alec and others},
  journal={arXiv preprint arXiv:2410.21276},
  year={2024}
}

@article{team2024gemini,
  title={Gemini 1.5: Unlocking multimodal understanding across millions of tokens of context},
  author={Team, Gemini and Georgiev, Petko and Lei, Ving Ian and Burnell, Ryan and Bai, Libin and Gulati, Anmol and Tanzer, Garrett and Vincent, Damien and Pan, Zhufeng and Wang, Shibo and others},
  journal={arXiv preprint arXiv:2403.05530},
  year={2024}
}

@article{xu2025qwen3,
  title={Qwen3-Omni Technical Report},
  author={Xu, Jin and Guo, Zhifang and Hu, Hangrui and Chu, Yunfei and Wang, Xiong and He, Jinzheng and Wang, Yuxuan and Shi, Xian and He, Ting and Zhu, Xinfa and others},
  journal={arXiv preprint arXiv:2509.17765},
  year={2025}
}

@article{cui2026minicpm,
  title={MiniCPM-o 4.5: Towards Real-Time Full-Duplex Omni-Modal Interaction},
  author={Cui, Junbo and Xu, Bokai and Wang, Chongyi and Yu, Tianyu and Sun, Weiyue and Xu, Yingjing and Wang, Tianran and He, Zhihui and Ma, Wenshuo and Cai, Tianchi and others},
  journal={arXiv preprint arXiv:2604.27393},
  year={2026}
}

@article{xie2024mini,
  title={Mini-Omni2: Towards Open-Source GPT-4o with Vision, Speech and Duplex Capabilities},
  author={Xie, Zhifei and Wu, Changqiao},
  journal={arXiv preprint arXiv:2410.11190},
  year={2024}
}

@article{li2024baichuan,
  title={Baichuan-Omni Technical Report},
  author={Li, Yadong and Sun, Haoze and Lin, Mingan and Li, Tianpeng and Dong, Guosheng and Zhang, Tao and Ding, Bowen and Song, Wei and Cheng, Zhenglin and Huo, Yuqi and others},
  journal={arXiv preprint arXiv:2410.08565},
  year={2024}
}

@inproceedings{tan2024judgebench,
  title={JudgeBench: A Benchmark for Evaluating LLM-Based Judges},
  author={Tan, Sijun and Zhuang, Siyuan and Montgomery, Kyle and Tang, William Y and Cuadron, Alejandro and Wang, Chenguang and Popa, Raluca Ada and Stoica, Ion},
  booktitle={The Thirteenth International Conference on Learning Representations},
  year={2025}
}

@article{son2024mm,
  title={MM-Eval: A Multilingual Meta-Evaluation Benchmark for LLM-as-a-Judge and Reward Models},
  author={Son, Guijin and Yoon, Dongkeun and Suk, Juyoung and Aula-Blasco, Javier and Aslan, Mano and Kim, Vu Trong and Islam, Shayekh Bin and Prats-Cristi{\`a}, Jaume and Tormo-Ba{\~n}uelos, Luc{\'\i}a and Kim, Seungone},
  journal={arXiv preprint arXiv:2410.17578},
  year={2024}
}
